\documentclass[draftcls,onecolumn,12pt]{IEEEtran}

\usepackage{amsmath}

\usepackage{pifont} % for checkmark and cross
\newcommand{\cmark}{\ding{51}} % ✓
\newcommand{\xmark}{\ding{55}} % ✗

\usepackage{amsfonts}
\usepackage{algorithm}%
\usepackage{algpseudocode}%

\usepackage{array}
\usepackage{textcomp}
\usepackage{stfloats}
\usepackage{url}
\usepackage{verbatim}
\usepackage{graphicx}
\def\BibTeX{{\rm B\kern-.05em{\sc i\kern-.025em b}\kern-.08em
    T\kern-.1667em\lower.7ex\hbox{E}\kern-.125emX}}
\usepackage{balance}
\usepackage{cite}
\DeclareGraphicsExtensions{.JPG,.eps,.pdf}
\usepackage{amsthm}

\usepackage{epstopdf}
\usepackage{color}
\usepackage{amscd}
\usepackage{soul}
\usepackage{stfloats}
\usepackage{verbatim}
\usepackage{makecell}
\usepackage{atbegshi}
\usepackage{flushend}
\usepackage{tikz,pgfplots,xcolor}

\usepackage{dsfont}
\usepackage{amssymb}
\usepackage{float}
\usepackage{mathtools}
\DeclarePairedDelimiter\norm{\lVert}{\rVert}

\usepackage{amsmath,amssymb,bm}

\usepackage{bm}
\usepackage{xcolor}
\usepackage{soul}
\newcommand{\Ex}{\mathbb{E}}
\newcommand{\bmtheta}{\bm{\theta}}

\newcommand{\vecm}{\mathbf{m}}
\newcommand{\vecx}{\mathbf{x}}
\newcommand{\vecy}{\mathbf{y}}

\newcommand{\vecw}{\mathbf{w}}

\newcommand{\vecs}{\mathbf{s}}


\DeclareMathOperator*{\argmin}{arg\,min}

\newtheorem{prop}{Proposition}
\newtheorem{assum}{Assumption}
\newtheorem{theorem}{Theorem}
\newtheorem{corollary}{Corollary}
\newtheorem{lemma}{Lemma}
\newtheorem{definition}{Definition}

\makeatletter
\newcommand{\removelatexerror}{\let\@latex@error\@gobble}

\begin{document}
%\ninept
\title{Noise Resilient Over-The-Air Federated Learning In Heterogeneous Wireless Networks}

\author{Zubair Shaban, Nazreen Shah, Ranjitha Prasad,~\IEEEmembership{Member,~IEEE}\\
        % <-this % stops a space
\noindent \thanks{We acknowledge the grants from DRDO CARS, SERB-FICCI PM Fellowship for Doctoral Research and LightMetrics Pvt. Ltd. A preliminary version of this work has been accepted for presentation at IEEE ICASSP, 2025.}% <-this % stops a space
\thanks{Zubair Shaban, Nazreen Shah, and Ranjitha Prasad are with the ECE dept., IIIT Delhi, New Delhi.}}

\maketitle



\maketitle

\begin{abstract}
Federated learning (FL) has emerged as a key enabler of distributed model training across ubiquitous wireless edge devices while preserving data privacy in next-generation wireless systems. One communication-efficient approach to realize this is leveraging Over-the-Air Federated Learning (OTA-FL), which exploits the superposition property of wireless multiple access channels (MAC) to perform low-latency model aggregation. However, practical OTA-FL deployments are significantly affected by the joint effects of channel noise at the server, fading, and both data and system heterogeneity across participating edge devices, which can degrade the model performance and thereby convergence. In this work, we propose the novel Robust and Convergent Federated Learning (RoCo-FL), a framework that jointly handles these challenges in federated wireless networks. RoCo-FL formulates local updates through inexact optimization, capturing the effects of wireless noise, partial participation, and heterogeneous system capabilities via structured inexactness conditions. These conditions naturally manifest as a proximal constraint at the edge devices that stabilizes local updates and mitigates the impact of channel impairments, data heterogeneity, and stragglers. We establish convergence guarantees for non-convex objectives under the proposed zeroth- and first-order inexactness models, providing theoretical insight into the interplay between channel impairments, data, and system heterogeneity. Experimentally, we validate RoCo-FL on real-world datasets, including CIFAR10, CIFAR100, and FEMNIST, demonstrating its robustness in noisy and heterogeneous environments. Compared to state-of-the-art baselines, RoCo-FL achieves significantly more stable convergence and higher accuracy, particularly in the presence of stragglers.
\end{abstract}

\section{Introduction}\label{sec1}
 \noindent In the current technology-driven era, the widespread deployment of wireless devices has enabled seamless connectivity and enhanced mobility, driving data-intensive applications envisioned for sixth-generation (6G) communication systems, such as Internet of Things (IoT) and edge computing. This evolution has intensified the need for synergy between wireless communications and data-driven machine learning (ML), with a growing emphasis on shifting ML workloads from centralized cloud infrastructures to the network edge to leverage local computation and flexible connectivity \cite{6GIOT,yang2022federated,gafni2022federated}. Federated Learning (FL) has emerged as a distributed, privacy-preserving ML paradigm that enables multiple edge users to collaboratively train a global ML model at the server while retaining the data at the edge user \cite{mcmahan2017communication}. Since the FL process relies heavily on the transmission and reception of the model parameters over the wireless channels (uplink and downlink) \cite{kairouz2021advances}, a key challenge is the efficient allocation of uplink channel bandwidth among the devices/clients.
 
  \noindent  Over-The-Air Federated Learning (OTA-FL) strategy becomes a preferred choice for efficient communication over a common uplink Multiple Access Channel (MAC)\cite{yang2020federated,sery2021over,OTAEarlyAccess}. OTA-FL leverages the technique of Over-the-Air Computation (AirComp) \cite{zhu2018broadband,amiri2020machine}, wherein the multiple access wireless channel functions as a natural aggregator. This approach enables clients to transmit updates simultaneously using analog signaling over the uplink channel in a non-orthogonal manner, effectively optimizing available temporal and spectral resources. However, OTA-FL is sensitive to the intrinsic Additive White Gaussian Noise (AWGN) and channel fading, which pose a challenge to the robustness of the global model, as demonstrated in Fig.~\ref{fig:IntroFig1}. In particular, fading results in dropped clients across communication rounds due to channel outages, which can further impede the progress of FL, potentially prolonging convergence\cite{mitra2021linear}. Several FL protocols in the literature operate under the assumption of an error-free channel, overlooking the inherent unreliability of wireless communication networks\cite{khan2021federated}. 
\begin{figure}[htbp]
    \centering
            \includegraphics[width=7.5cm,height=6cm]{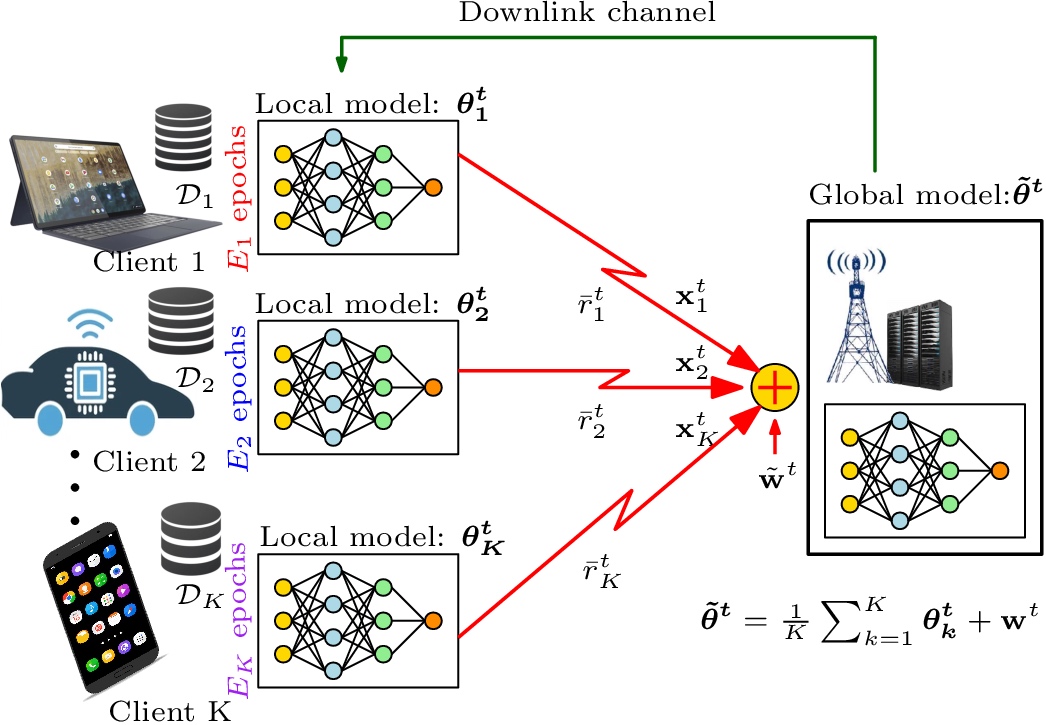}
      \hspace{5mm}
      \includegraphics[width=7.5cm,height=4.5cm]{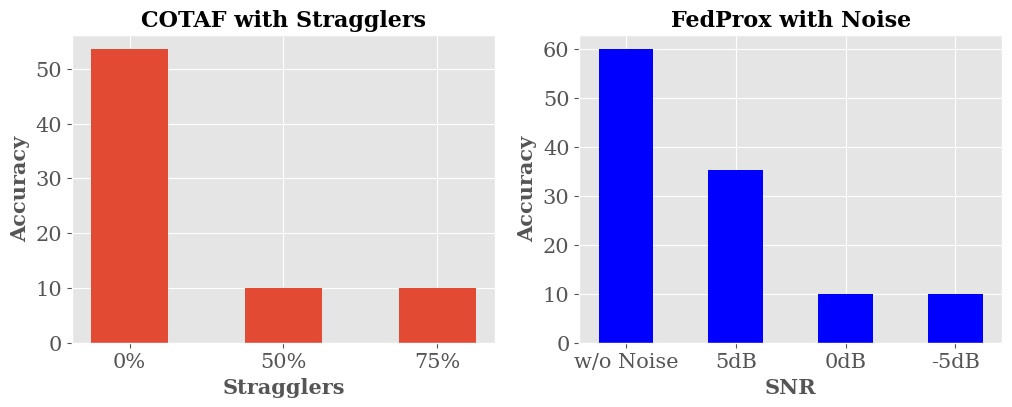}
      \caption{Left: Practical OTA-FL under data and system heterogeneity with wireless channel impairments. Right: Bar plots showing (i) the impact of system heterogeneity (stragglers) on COTAF \cite{sery2021over} and (ii) the impact of noise on FedProx \cite{li2020federated}, both on CIFAR-10. COTAF degrades in the presence of stragglers, while FedProx fails at low SNR, motivating a method that jointly addresses system heterogeneity and noisy wireless environments.}
     \label{fig:IntroFig1}
\end{figure}

\noindent Concurrently, a critical challenge in FL is the data and system heterogeneity \cite{10373878, li2020federated}. Data heterogeneity stems from non-independent and identically distributed (non-IID) data distributions available at the edge devices, leading to variability in data volume, label imbalance, and model performance, which further leads to system heterogeneity \cite{kairouz2021advances}. Furthermore, system heterogeneity manifests in the variable computational capabilities among participating clients, giving rise to the notorious \emph{straggler effect}\cite{ozfatura2020straggler}. This phenomenon occurs when the FL process slows down as dictated by the pace of the slowest client due to underlying reasons such as local memory access, limited computational capabilities, and background processes. In wireless FL, a common solution to deal with the straggler effect is to simply ignore the updates from the slower clients, as done in \cite{sery2021over}. However, ignoring updates from the slower clients results in the client drift problem, wherein the global model starts to favor the local solutions of faster clients.  Hence, despite the undeniable potential of FL in wireless environments, there are robustness issues in practice that arise due to the joint effect of heterogeneity, noise, and fading \cite{ang2020robust}.
\noindent Recent advances in OTA-FL have also explored 
several directions to address the 
limitations of existing approaches, such as \cite{xiao2024over} provide 
a comprehensive survey of OTA-FL, 
identifying open challenges, including joint handling of noise, heterogeneity, 
and resource constraints as a critical 
unsolved problem. Approaches that mitigate system heterogeneity in FL include FedProx \cite{li2020federated}, FedNova \cite{wang2020tackling}, and SCAFFOLD \cite{karimireddy2020scaffold}. FedProx addresses heterogeneity via inexact proximal updates, while FedNova mitigates stragglers by allowing faster clients to perform additional local updates. SCAFFOLD introduces control variates to reduce client drift caused by data heterogeneity. However, these methods are not designed for wireless environments and do not account for channel noise or fading.
\noindent In wireless FL, COTAF \cite{sery2021over} employs over-the-air precoding to mitigate noise, while one related non-FL work, LASER \cite{pmlr-v235-makkuva24a_laser}, focuses on communication compression under noisy channels. Nevertheless, both approaches overlook data and system heterogeneity. Similarly, robust FL methods such as \cite{ang2020robust} incorporate regularization to combat noise but fail to jointly address noise, fading, and heterogeneity. As a result, straggler effects and client drift under realistic wireless conditions remain largely unresolved. In addition, \cite{10388035} provides 
a unified convergence analysis of 
AirComp-based FedAvg under both strongly 
convex and non-convex objectives with 
data heterogeneity, establishing 
performance benchmarks for OTA-FL 
over fading channels. 
Despite these advances, none of the above works simultaneously address AWGN, fading, data heterogeneity, 
and system heterogeneity with provable 
convergence guarantees for non-convex 
objectives, which is the focus of 
this work.
\subsection{Contributions} We make the following contributions in this work:
\begin{itemize}
    \item \textbf{Noise-resilient OTA-FL framework:}
    We propose \emph{Noise Resilient Over-the-Air Federated Learning (RoCo-FL)}, a wireless FL framework that jointly addresses system heterogeneity (stragglers), data heterogeneity, and wireless impairments such as noise and fading by solving a robust constrained local optimization problem at the clients, followed by model aggregation at the server. To mitigate the impact of uplink noise, a novel precoding factor is introduced. In fading environments, RoCo-FL incorporates a threshold-based client selection strategy and performs \emph{partial-participation} aggregation, enabling robust learning despite channel outages.

    \item \textbf{Convergence guarantees under non-convexity and inexactness:}
We establish convergence guarantees for RoCo-FL with \emph{non-convex} local objectives.
Our analysis explicitly models first- and zeroth-order inexactness arising from wireless noise, fading, data heterogeneity, and partial local computation, thereby accommodating straggling and intermittently participating clients.
Under the $(B,H)$-Local Gradient Dissimilarity (LGD) assumption, we show that RoCo-FL exhibits an $\mathcal{O}(1/T)$ transient and converges to a neighborhood of first-order stationary points, whose size explicitly depends on data heterogeneity and wireless communication conditions.
When $H=0$, the neighborhood vanishes and the result recovers exact $\mathcal{O}(1/T)$ first-order stationarity.
Moreover, without relying on the LGD assumption, we establish convergence under a zeroth-order inexactness model with an $\mathcal{O}(1/\sqrt{T})$ rate and linear speedup in the number of participating clients.
To the best of our knowledge, this is the first convergence analysis of OTA-based FL under $(B,H)$-LGD and non-convex objectives that jointly captures wireless impairments, partial participation, and inexact local computation.

    \item \textbf{Empirical validation:}
    We validate the effectiveness of RoCo-FL through experiments on FEMNIST, CIFAR-10, and CIFAR-100, demonstrating robust performance under AWGN and fading channels combined with varying levels of data and system heterogeneity. The results show that RoCo-FL achieves comparable or superior convergence behavior relative to existing methods while explicitly accounting for wireless impairments and straggler effects.
\end{itemize}
A detailed comparison of assumptions and worst-case convergence rates with existing methods is summarized in Table~\ref{tab:method_comparison}. The table highlights RoCo-FL as the only framework that jointly addresses practical wireless constraints and system heterogeneity with provable convergence guarantees in realistic non-convex FL scenarios. Note that $p$ in the FedPD rate depends on data heterogeneity. The remainder of the paper is organized as follows. Section~\ref{sec:problem_statement} presents the system model and problem formulation, Section~\ref{sec:convergence_analysis} and~\ref{sec:zeta_convergence} provide the convergence analysis, Section~\ref{sec:experiments} reports experimental results, and Section~\ref{sec:conc} concludes the paper.

\section{Problem Statement and Proposed Algorithm}
\label{sec:problem_statement}

%, i.e., the total number of instances is given by $D  = \sum_{k = 1}^K D_k$
\noindent We consider a wireless multi-user FL framework with a central server communicating with $K$ wireless clients. The local model updates from the clients to the server occur over the resource-constrained uplink channel, and the global synchronization between the server and clients occurs over the downlink channel \cite{amiri2020machine}.  Each client $k$ possesses a local dataset $\mathcal{D}_k$ of size $D_k$, with $\sum_{k=1}^K D_k = D$.  In supervised learning, $\mathcal{D}_k$ consists of data samples as a set of input-output pairs $\{\vecs_{k,j}, y_{k,j}\} \in \mathcal{D}_k$ where $y_{k,j} \in \mathbb{R}$ is the label for the sample $\vecs_{k,j}$ for $j \in [D_k]=\{1,2,...,D_k\}$. The goal of the local ML problem is to estimate the model parameters $\bmtheta_k \in \mathbb{R}^d$, by optimizing the empirical loss function given by
\begin{equation}
    f_k(\bmtheta_k) \triangleq \frac{1}{D_k} \sum_{j=1}^{D_k} l(\vecs_{k,j},y_{k,j};\bmtheta_k),
\end{equation}
where $l(.;.)$ is the user-specified loss function that quantifies the discrepancy between the predicted output and the ground truth for each sample in the local dataset. In FL, the goal is to obtain a global model, $\bmtheta$, which minimizes a loss function $ F(\bmtheta)$ given as
\begin{equation}
    \text{P0}:~~\min_{\bmtheta}\left\{ F(\bmtheta) \triangleq\frac{1}{K}\sum_{k=1}^K f_k(\bmtheta)\right\}.
    \label{eq:P0}
\end{equation}
Evidently, P0 does not require any access to a dataset, and it allows for local learning via $f_k(\bmtheta)$ and subsequently estimating the global model update via aggregation \cite{mcmahan2017communication}.

\noindent A promising approach to address the resource allocation challenges in multi-client systems is to leverage the capabilities of OTA-FL communications \cite{xiao2024over, sery2021over}. In OTA-FL, the clients transmit their model updates over the shared wireless MAC via analog signaling, enabling simultaneous access to both time and spectral resources \cite{OTA_demo}. Although both the uplink and downlink channels are noisy by nature, the effect of noise is considered only in the uplink since the noise in the downlink broadcast channel can be compensated using sophisticated error coding schemes. We consider that all clients adopt OTA to communicate local models with the server with perfect synchronization, which can be realized by the timing advance mechanism \cite{8930825,book}, widely
adopted in 5G new radio (NR). Each client $k \in [K]$ transmits $\vecx_k^t \in \mathds{R}^d$ in the $t$-th communication round, where $t \in [T]$. The channel output at the server is given as:
\begin{equation}
\vecy^t=\sum_{k=1}^{K}\vecx_k^t+{\mathbf{\tilde w}}^t,\quad \text{where} \quad \Ex[\lVert \mathbf{x}_k^t \rVert^2]\leq P.
\label{eq:power}
\end{equation}
Here, ${\mathbf{\tilde w}}^t \sim \mathcal{N}(0,\,\sigma^{2} \mathbf{I}_d)$ is the AWGN in  the uplink MAC channel. Further, $\Ex[\cdot]$ represents the expectation and $P>0$ is the available transmission power.
% Hence, the nature of the MAC channel allows direct access to the sum of the client updates at the server and not the individual updates. 
In the context of OTA-FL, the local model update at the $k$-th client in the $t$-th round is precoded as $\mathbf{x}_k^{t}$ \cite{sery2021over}, which acts as an input to the MAC. The precoded vector $\mathbf{x}_k^{t}$ is given as
\begin{align}
    \mathbf{x}_k^{t}=\sqrt{ p^{t}}(\bmtheta_k^{t}-\tilde\bmtheta^{t-1}),
    \label{eq:precoder1}
\end{align}
where the precoding factor $p^{t} \triangleq \tfrac{P}{ \sum_{k=1}^K q_k\{\Ex[\lVert \bmtheta_k^{t}-\tilde\bmtheta^{t-1}\rVert^2]\}}$ with the weights $q_k=D_k/D$, $\bmtheta_k^{t}$ is the local model update at $t$-th round and $\tilde \bmtheta^{t-1}$ denotes the noisy global update from the previous round. The server recovers the global aggregated model {$\tilde{\bmtheta}^{t}$ } from the received signal $\vecy^t$ at the server using the following decoding rule:
\begin{align}
    \tilde{\bmtheta}^{t} =\dfrac{\vecy^{t}}{K\sqrt{ p^{t}}}+\tilde{\bmtheta}^{t-1} = \frac{1}{K}\sum_{k=1}^K \bmtheta_k^{t} + \vecw^{t},
    \label{eq:global}
        \end{align}
% $\bmtheta^{t-1}=\frac{1}{K}\sum_{k=1}^{K}\bmtheta_k^{t-1}$ and 
where $\vecw^{t}\triangleq\frac{\tilde{\vecw}^{t}}{K\sqrt {p^{t}}}$, i.e., $\vecw^{t} \sim \mathcal{N} (0,\frac{\sigma^2}{K^2 p^{t}}\mathbf{I}_d)$. From \eqref{eq:global}, it is evident that in OTA communications, the aggregated ML model is susceptible to corruption due to AWGN, a factor that has the potential to undermine the convergence and performance of learning algorithms. 

To mitigate straggler effects, we consider the partial participation case, where a subset of clients, $\hat{K} < K$, participate in model aggregation at the server. In such a scenario, the decoding rule is given as
\begin{align}
    \tilde{\bmtheta}^{t} =\dfrac{\vecy^{t}}{\hat K\sqrt{ p^{t}}}+\tilde{\bmtheta}^{t-1} = \frac{1}{\hat K}\sum_{k \in \mathcal{S}^{t}} \bmtheta_k^{t} + \frac{K}{\hat K}\vecw^{t},
    \label{eq:global_partial}
\end{align}
where $\hat K = |\mathcal{S}^{t}|$ is the cardinality of the set of clients, $\mathcal{S}^{t}$, chosen in each round. Effectively, the noise in such a scenario is given by $\hat\vecw^t \triangleq\frac{K}{\hat K}\vecw^t$, i.e., $\hat\vecw^t \sim \mathcal{N} (0,\frac{\sigma^2}{\hat K^2 p^t}\mathbf{I}_d)$.  

\subsection{Handling Data and System Heterogeneity}
\label{sec:datasyshet}

\noindent Federated learning systems operate under significant heterogeneity arising from differences in client computation, energy constraints, and unreliable wireless links affected by noise and fading. Enforcing all clients to complete a fixed number of local epochs therefore leads to severe straggler effects~\cite{sery2021over}. To ensure scalability and robustness, it is necessary to relax the requirement of exact local optimization and allow clients to return approximate solutions. To this end, we adopt a proximal regularization framework that modifies the local objective as
\[
h_k(\bmtheta;\tilde\bmtheta^{t-1}) = f_k(\bmtheta) + \frac{\lambda}{2}\|\bmtheta-\tilde\bmtheta^{t-1}\|^2,
\]
which encourages local solutions to remain close to the (possibly noisy) global model while tolerating partial computation. This leads to the following global optimization problem:
\begin{align}
    \text{P1:}~\min_{\bmtheta}\;
    F(\bmtheta) \triangleq \frac{1}{K}\sum_{k=1}^K h_k(\bmtheta;\tilde\bmtheta^{t-1}).
    \label{eq:P1}
\end{align}

\noindent Solving P1 exactly at every client is impractical in heterogeneous wireless environments. Hence, we characterize local solutions using standard notions of \emph{inexactness}~\cite{li2020federated,yuan2022convergence_newfedprox}, summarized below.

\begin{definition}(Inexact solution)~ Given  $\lambda>0$, for a differentiable function $h_k(\bmtheta;\tilde\bmtheta^{t-1})=f_k(\bmtheta)+\frac{\lambda}{2}\norm{\bmtheta-\tilde\bmtheta^{t-1}}^2$, 
    \begin{enumerate}
        \item $\bmtheta^t_k$ is a $\gamma_k^t$-inexact solution of $\min\limits_{\bmtheta}h_k(\bmtheta;\tilde\bmtheta^{t-1})$ if $\norm{\nabla h_k(\bmtheta^t_k;\tilde\bmtheta^{t-1})}\leq \gamma_k^{t} \norm{\nabla{f}_k(\tilde\bmtheta^{t-1})} + \norm{\psi^t}$, where $\gamma_k^t \in [0,1]$ and $\psi^t$ depends on the channel parameters.
        \item $\bmtheta^t_k$ is a $\zeta^t$-inexact solution of $\min\limits_{\bmtheta}h_k(\bmtheta;\tilde\bmtheta^{t-1})$ if $h_k(\bmtheta^t_k;\tilde\bmtheta^{t-1})\leq\min\limits_{\bmtheta}h_k(\bmtheta,\tilde\bmtheta^{t-1})+\zeta^t$,  for $\zeta^t\geq0$.
    \end{enumerate}
    \label{defn1}
\end{definition}
 
\noindent The first notion, termed as 
\emph{first-order inexactness}, allows 
clients to terminate local optimization 
at approximate stationary points, 
enabling stragglers to contribute partial updates under noisy conditions. 
In RoCo-FL, this inexactness has two distinct and concurrent causes: $a)$ the OTA aggregation noise corrupts the global model $\tilde{\theta}^{t-1}$ 
broadcast by the server, so every client initializes local optimization from a 
noisy reference point, regardless of its computational capability, and $b)$ stragglers or fading-induced dropouts force clients 
to terminate local optimization early. 
Together, these make an inexact local 
solutions an inevitable feature of 
wireless FL rather than a modeling 
convenience. 
The parameter $\gamma_k^t$ serves as a 
proxy for the amount of local computation 
performed by client $k$ at round $t$. 
A smaller $\gamma_k^t$ indicates a more 
exact solution corresponding to more 
local computation, while a larger 
$\gamma_k^t$ indicates fewer local steps 
or limited computational resources.  
This flexibility is essential in the heterogeneous wireless setting, where clients differ in battery state, computational capability, 
and channel conditions, hence making it impractical to 
enforce a uniform computation requirement 
across all clients. The additional term 
$\|\psi^t\|$ quantifies the additional 
inexactness induced by the corrupted 
global model $\tilde{\theta}^{t-1}$ 
received by the clients in the presence of OTA noise. 
The second notion, \emph{zeroth-order inexactness}, bounds suboptimality in 
objective value and is particularly 
suitable for severely resource-constrained 
clients. 
Together, these definitions provide a 
unified framework for modeling 
heterogeneity and communication 
impairments in wireless FL.

\subsection{Problem Formulation in the Presence of Fading}
\label{sec:fading}

\noindent We consider an OTA-FL system where client $k$ experiences a block-fading uplink channel in round $t$, modeled as
$\bar r_k^t = r_k^t e^{j\Omega_k^t}$, where $r^t_k$ represents the magnitude and $\Omega_k^t$ is the phase induced due to fading. The received signal over the MAC is given as 
% \vspace{-1em}
\begin{align}
\vecy^t = \sum_{k=1}^K r_k^t e^{j\Omega_k^t} \vecx_k^t + \tilde{\vecw}^t,
\label{eq:FadingOTA}
\end{align}
where $\tilde{\vecw}^t$ is additive Gaussian noise. To mitigate severe attenuation when $r_k^t$ is small, we adopt the threshold-based client selection as in \cite{sery2021over}, where only clients with $r_k^t > \hat r$ transmit their updates. Assuming the channel state information (CSI) is available at the clients, the transmitted signal from each client at round $t$ is expressed as
\begin{align}
\vecx_k^t =
\begin{cases}
\dfrac{\hat r \sqrt{p^t}}{r_k^t} e^{-j\Omega_k^t} (\bmtheta_k^t - \tilde{\bmtheta}^{t-1}), & r_k^t > \hat r, \\
0, & \text{otherwise}.
\end{cases}
\label{eq:fading_equation}
\end{align}

\noindent Let $\mathcal{K}^t$ denote the set of participating clients in round $t$. The server-side OTA aggregation yields
\begin{align}
\tilde{\bmtheta}^t
= \frac{1}{|\mathcal{K}^t|}
\sum_{k \in \mathcal{K}^t} \bmtheta_k^t + \bar{\vecw}^t,
\label{eq:global_fading}
\end{align}
where $\bar\vecw^t \triangleq\frac{K}{\hat r|\mathcal{K}^t| }\vecw^t$ and $\bar\vecw^t \sim \mathcal{N} (0,\frac{\sigma^2}{\hat r^2|\mathcal{K}^t|^2 p^t  }\mathbf{I}_d)$. This leads to a partial-participation FL problem, where the global model in round $t$ solves
\begin{align}
\text{P2:}\quad
\min_{\bmtheta}\;
F(\bmtheta)
\triangleq
\frac{1}{|\mathcal{K}^t|}
\sum_{k \in \mathcal{K}^t}
h_k(\bmtheta;\tilde{\bmtheta}^{t-1}),
\label{eq:P2}
\end{align}
where $
h_k(\bmtheta;\tilde{\bmtheta}^{t-1})
=
f_k(\bmtheta)
+
\frac{\lambda}{2}
\|\bmtheta - \tilde{\bmtheta}^{t-1}\|^2.$

\subsection{Noisy SGD and Proximal Constraint}
\label{sec:Prox}

\noindent In this subsection, we show that when Stochastic Gradient Descent (SGD) is used as the local solver, a proximal constraint naturally arises as a robust design choice in the presence of communication noise. In OTA-FL, noise and fading corrupt the global model at the server, which is then broadcast to the clients. Let the noisy global model at communication round $t-1$ be denoted by $\tilde{\bmtheta}^{t-1}$, modeled as
\begin{equation}
    \tilde{\bmtheta}^{t-1} = \bmtheta^{t-1} + \Delta \bmtheta^{t-1},
    \label{eq:noisyEstimate}
\end{equation}
where $\Delta \bmtheta^{t-1}$ captures the aggregate effect of channel noise and fading. Each client initializes its local optimization using $\tilde{\bmtheta}^{t-1}$. When SGD is employed, the one-step local update at client $k$ is given by
\begin{align}
  \bmtheta_k^{t} = \tilde{\bmtheta}^{t-1} - \eta^t \nabla f_k(\tilde{\bmtheta}^{t-1}),
   \label{eq:noiseEffectSGD}
\end{align}
where $\eta^t$ denotes the learning rate.

\noindent Let $g_k(\cdot) \triangleq \nabla f_k(\cdot)$ and noiseless one-step SGD update $\check{\bmtheta}_k^{t} = \bmtheta^{t-1} - \eta^t g_k(\bmtheta^{t-1})$. Combining \eqref{eq:noisyEstimate} and \eqref{eq:noiseEffectSGD}, we obtain
\begin{align}
    \bmtheta_k^{t} 
    &= \check{\bmtheta}_k^{t} + \Delta \bmtheta^{t-1} - \eta^t \Delta g_k(\bmtheta^{t-1})=\check{\bmtheta}_k^{t} + \Lambda,
    \label{eq:finalnoiseterms}
\end{align}
where $\Delta g_k(\bmtheta^{t-1}) = g_k(\tilde{\bmtheta}^{t-1}) - g_k(\bmtheta^{t-1})$ and $\Lambda \triangleq \Delta \bmtheta^{t-1} - \eta^t \Delta g_k(\bmtheta^{t-1})$. This shows that both corrupted initialization and noisy gradient evaluation contribute to deviations from the noiseless update.

\noindent Stochastic optimization theory \cite{6796505} suggests that robustness to additive noise can be improved by augmenting the local loss with a regularization term. Accordingly, the local objective can be approximated as $
    \bar{f}_k(\bmtheta) = f_k(\bmtheta) + \tilde{\lambda} \| g_k(\bmtheta^{t-1}) \|^2,$ where $\tilde{\lambda} = \sigma_{\Lambda}^2 (\eta^t)^2$ and $\sigma_{\Lambda}^2$ denotes the variance of $\Lambda$. Although the noise terms in \eqref{eq:finalnoiseterms} can be modeled as Gaussian, accurately estimating their variances remains challenging in practice \cite{elbir2021federated}.

\noindent Using a first-order approximation relating gradients to parameter differences, the regularized objective can be written as
\begin{align}
   \bar{f}_k(\bmtheta) = f_k(\bmtheta) + \tilde{\lambda} \| \bmtheta - \bmtheta^{t-1} \|^2,
   \label{eq:proximalNoise}
\end{align}
where $\tilde{\lambda}$ is treated as a tunable hyperparameter. Notably, \eqref{eq:proximalNoise} coincides with the proximal surrogate $h_k(\bmtheta;\tilde{\bmtheta}^{t-1})$ introduced in Problems~P1 and~P2. This observation justifies the proximal formulation adopted in this work, which simultaneously addresses data heterogeneity and communication noise, providing a unified optimization framework for wireless FL.

\subsection{Proposed Algorithm: RoCo-FL}

\noindent In this section, we present \emph{RoCo-FL}, an algorithmic framework that operationalizes the OTA-FL formulations introduced in Problems P1 and P2. RoCo-FL explicitly accounts for communication noise, fading, and client heterogeneity by solving proximal-constrained local subproblems and aggregating the resulting updates over the wireless channel.

\noindent Unlike the baseline OTA-FL solution corresponding to P0, which reduces to federated averaging with precoding as in COTAF~\cite{sery2021over}, the formulation in P1 incorporates a proximal regularization term that captures the effect of noise in each communication round. As discussed in Sec .~\ref{sec:datasyshet}, this proximal constraint simultaneously mitigates data and system heterogeneity by allowing inexact local solutions, thereby enabling the inclusion of straggling and resource-constrained clients.

\noindent RoCo-FL provides a general framework for solving constrained local optimization problems while optimizing the global objective $F(\bm{\theta})$ at the server. The framework is agnostic to the choice of local solver, provided that the resulting local updates satisfy the inexactness conditions described in Definition~\ref{defn1}. This flexibility allows RoCo-FL to be instantiated using a wide range of optimization methods. 

\noindent Algorithm~\ref{tab:AlgTable} presents a concrete instantiation of RoCo-FL when SGD is employed as the local solver. At each communication round $t$, client $k$ performs $E_k \leq E$ local epochs, where $E$ denotes the maximum number of local epochs. The value of $E_k$ may vary across clients depending on their computational capabilities, channel conditions, or energy constraints, naturally capturing system heterogeneity within the proposed framework. The inexactness conditions in Definition~\ref{defn1} are guaranteed by the local computations rather than enforced explicitly.

\begin{algorithm}[htbp]
\caption{RoCo-FL (SGD)}
\label{tab:AlgTable}

\begin{algorithmic}[1]

\Require $K$, $\hat K$, $E$, $T$, $\lambda$, $\gamma^t$, $\hat r$.

\State Initialize global model $\tilde{\bmtheta}^0$.

\For{$t=1,\ldots,T$}

    \State \textbf{Client-side (in parallel):}

    \State Each client $k$ sets
    $\bmtheta_k^{t-1}=\tilde{\bmtheta}^{t-1}$
    and computes $p^t$.

    \State Perform $E_k\le E$ local SGD steps:

    \State
    $\bmtheta_{k}^{t}
    \leftarrow
    \bmtheta_{k}^{t-1}
    -\eta^t
    \nabla h_k(\bmtheta_{k}^{t-1},
    \tilde{\bmtheta}^{t-1})$.

    \State Transmit $\mathbf{x}_k^t$ using the appropriate OTA model (with or without fading).

    \State \textbf{Server-side:}

    \State Aggregate received signals to recover $\tilde{\bmtheta}^t$ and broadcast it to all clients.

\EndFor

\Ensure $\bmtheta^T$

\end{algorithmic}
\end{algorithm}

\section{Convergence Analysis:  $\gamma^{\lowercase{t}}$-inexactness}
\label{sec:convergence_analysis}
\noindent We demonstrate the convergence of RoCo-FL, which provides robust parameter estimates in the presence of wireless channel noise, fading effects, and heterogeneity. Unlike other works, we do not restrict our analysis to convex or strongly convex functions and prove convergence for non-convex functions. In the sequel, $\Ex_k[\cdot] = \sum_{k=1}^K q_k (\cdot)$ where the weights are given by $q_k={D_k}/{D}$ such that $\sum_{k=1}^K q_k =1$. We use the following assumptions for all the clients:
\begin{assum}
    ($(B, H)$-LGD)
The local functions $f_k(\cdot)$ are \((B, H)\)-local gradient dissimilar (LGD) if the following holds for all $\bmtheta \in \mathbb{R}^d$:
\[
\Ex_k\|\nabla f_k(\bmtheta)\|^2 \leq B^2 \|\nabla F(\bmtheta)\|^2 + H^2.
\]
\end{assum}

{\begin{assum}
    ($L_0$-Local Lipschitz continuity) The  function \(F(\cdot):A \subset \mathbb{R}^d \rightarrow \mathbb{R}\) is locally Lipschitz at \(\bmtheta' \in A\)  if there exist constants \(\xi > 0\) and \(L_0 \in \mathbb{R} \) such that for each \(\bmtheta \in A, \norm*{\bmtheta - \bmtheta'}<  \xi \Rightarrow {F(\bmtheta)-F(\bmtheta')}\leq L_0 \norm*{\bmtheta - \bmtheta'}\).
\end{assum}}
\begin{assum}
    The local functions $f_k(\cdot):\mathbb{R}^d \rightarrow \mathbb{R}$ are non-convex and $L$-Lipschitz smooth.
\end{assum}
\noindent Prior to presenting the convergence analysis, we derive an upper bound on the precoding factor $p^t$.
\begin{lemma}
The precoding factor $p^t$ in each communication round can be upper bounded as $\frac{1}{p^t}\leq \frac{B^2 \norm{\nabla F(\tilde\bmtheta)}^2 +H^2}{P}$ under first-order inexactness, for  $\lambda\leq\frac{K^2\tau\bar\mu}{2}$ where $\tau=\frac{P}{d\sigma^2}$ denotes the Signal to Noise Ratio (SNR).
\label{lem:precodingfactor_1storder}
\end{lemma}

\subsection{Convergence under $(B,H)$ and $(B,0)$-LGD}

\noindent
We first establish convergence under the general $(B,H)$-LGD assumption, and subsequently recover the $(B,0)$-LGD results as special cases by setting $H=0$.

\noindent
Unlike conventional assumptions such as bounded variance, the $(B,0)$-LGD assumption captures data heterogeneity without imposing restrictive constraints on gradient norms or variance bounds~\cite{li2022convergence}. 
However, this assumption requires all local objectives to share a common stationary point with the global objective, i.e., $\|\nabla F(\bmtheta)\|=0$ implies $\|\nabla f_k(\bmtheta)\|=0$ for all $k\in[K]$, which is often unrealistic in FL.
To address this limitation, we consider Assumption~1 with $H>0$, where $\|\nabla F(\bmtheta)\|=0$ only ensures $\|\nabla f_k(\bmtheta)\|\le H$, allowing for distinct local optima.
This motivates our analysis under the more general and realistic $(B,H)$-LGD assumption, where $H$ explicitly quantifies the level of data heterogeneity.

\begin{theorem}[Partial participation under $(B,H)$-LGD]
\label{partial_participation_theorem_BH_M}
Let $\hat K<K$ clients in the set $\mathcal{S}^t\subset[K]$ participate in the $t$-th communication round and let $\Delta = F(\bmtheta^0)-F(\bmtheta^T)$.
Given the first-order inexactness in Definition~\ref{defn1}, Assumptions~$1$--$3$, $\bar\mu=\lambda-\bar L>0$ for $\bar L>0$ and $\hat \alpha>0$, the iterates generated by RoCo-FL satisfy
\begin{small}
\begin{align}
\frac{1}{T}\sum_{t=0}^{T-1}\mathbb{E}\!\left[\|\nabla F(\tilde{\bmtheta}^t)\|^2\right]
\le
\mathcal{O}\!\left(\frac{\Delta}{\hat\alpha T}\right)
+
\mathcal{O}\!\left(\frac{\hat \beta H^2}{\hat \alpha}\right),
\label{eq:partial_participation_theorem_BH_M}
\end{align}
\end{small}
\noindent where  $\hat\alpha=\hat\rho_1 -\frac{\hat C_1}{ \hat K^2\tau}$, $\hat\beta=\hat\rho_2 +\frac{\hat C_2}{\hat K^2\tau}$ and $\hat\rho_1, \hat C_1, \hat \rho_2, \hat C_2$ are deferred to the Appendix.
\end{theorem}
\noindent
Theorem~\ref{partial_participation_theorem_BH_M} establishes neighborhood convergence of RoCo-FL under partial participation and heterogeneous data. The bound in~\eqref{eq:partial_participation_theorem_BH_M} exhibits the canonical structure of non-convex federated optimization, consisting of a transient $\mathcal{O}(1/T)$ term and a residual term corresponding to convergence to a neighborhood of first-order stationary points.
The size of the convergence neighborhood is governed by the heterogeneity level $H$ through the normalized ratio $H^2/\hat\alpha$, capturing the extent to which client-specific gradients deviate from the global descent direction.
When $H=0$, the heterogeneity-induced bias vanishes, and the result reduces to exact $\mathcal{O}(1/T)$ convergence, recovering classical inexact FL behavior.
For heterogeneous data, the algorithm converges to a controlled neighborhood whose size scales gracefully with the degree of client drift.
Importantly, the descent coefficient $\hat\alpha$ explicitly reveals the impact of wireless communication quality.
Higher SNR $\tau$ and larger participation level $\hat K$ jointly enlarge $\hat\alpha$, strengthening descent and tightening the resulting neighborhood.
This highlights a key system-level insight, i.e., improved communication reliability mitigates the adverse effects of data heterogeneity by stabilizing over-the-air aggregation.
In contrast to prior neighborhood convergence results such as FedPD~\cite{FedPD}, which do not account for communication effects, our analysis explicitly incorporates OTA aggregation, partial participation, and channel noise, yielding convergence guarantees that directly inform the design of practical wireless FL systems.  

\begin{corollary}[Communication complexity for neighborhood stationarity]
\label{cor:epsilon_neighborhood}
Under the conditions of Theorem~\ref{partial_participation_theorem_BH_M},
define the stationarity floor as
$\epsilon_{\rm floor}\triangleq(\hat\beta/\hat\alpha)H^2$.
For any target accuracy $\epsilon>\epsilon_{\rm floor}$, RoCo-FL satisfies
$\frac{1}{T}\sum_{t=0}^{T-1}
\mathbb{E}\!\left[\|\nabla F(\tilde{\bmtheta}^{t})\|^2\right]
\leq \epsilon$
after
\begin{align}
    T
    \geq
    \frac{\Delta}
    {\hat\alpha\left(\epsilon-\epsilon_{\rm floor}\right)}
    =
    \mathcal{O}\!\left(
    \frac{\Delta}
    {\hat\alpha\epsilon-\hat\beta H^2}
    \right)
    \label{eq:communication_complexity_neighborhood}
\end{align}
communication rounds.
\end{corollary}

\noindent
\textbf{System-level implication:}
Theorem~\ref{partial_participation_theorem_BH_M} identifies two distinct
roles of the wireless system. First, the condition $\hat\alpha>0$
determines whether the communication quality and participation level are
sufficient to preserve expected descent. Second, once descent is
guaranteed, the ratio $\hat\beta/\hat\alpha$ determines how wireless
impairments interact with statistical heterogeneity in the resulting
stationarity neighborhood. Hence, communication reliability affects not
only whether RoCo-FL converges, but also the quality of the stationary
neighborhood attained under heterogeneous data.

% \noindent
% As a special case, we recover the classical $(B,0)$-LGD results by setting $H=0$.
\begin{corollary}[Partial participation under $(B,0)$-LGD]
\label{partial_participation_theorem}
Under the assumptions of Theorem~\ref{partial_participation_theorem_BH_M} with $H=0$, the iterates generated by RoCo-FL satisfy
\begin{small}
\begin{align}
\frac{1}{T}\sum_{t=0}^{T-1}\mathbb{E}\!\left[\|\nabla F(\tilde{\bmtheta}^t)\|^2\right]
\le
\frac{\Delta}{\hat\alpha T}.
\label{eq:partial_participation_eq}
\end{align}
\end{small}
Thus, under the stated conditions, RoCo-FL attains an
$\mathcal{O}(1/T)$ ergodic first-order stationarity rate when the additive
heterogeneity component vanishes i.e., when $H=0$.
\end{corollary}

\noindent\textbf{Fading as a Case of Partial Device Participation}
We now extend the above analysis to fading channels.
Clients compensate for channel fading using local CSI \cite{sery2021over} and inverse channel scaling, while a threshold $\hat r$ prevents excessive transmit power by excluding clients with weak channels.
This induces a structured form of partial participation.

\begin{corollary}[Fading under $(B,H)$-LGD]
\label{fading_theorem_BH}
Under the assumptions of Theorem~\ref{partial_participation_theorem_BH_M}, $\bar\alpha>0$, and $\hat r>0$, RoCo-FL satisfies
\begin{small}
\begin{align}
\frac{1}{T}\sum_{t=0}^{T-1}\mathbb{E}\!\left[\|\nabla F(\tilde{\bmtheta}^t)\|^2\right]
\le
\mathcal{O}\!\left(\frac{\Delta}{\bar\alpha T}\right)
+
\mathcal{O}\!\left(\frac{H^2}{\bar\alpha}\right),
\end{align}
\end{small}
where 
\begin{small}
$\bar\alpha=\hat \rho_1-\hat C_1\frac{d\sigma^2}{\hat r^2\hat K^2 P}$.
\end{small}
\end{corollary}
\noindent
The parameter $\hat r$ acts as a tunable mechanism that balances participation diversity and communication reliability.
Smaller values of $\hat r$ allow more clients to participate under fading, while larger values prioritize reliable aggregation from stronger channels.
Moreover, improved SNR enlarges the feasible range of $\hat r$, thereby enhancing convergence robustness in fading environments.

\section{Convergence Analysis: $\zeta^{\lowercase{t}}$-inexactness}\label{sec:zeta_convergence}
\noindent In this section, we provide the convergence analysis of RoCo-FL under the notion of $\zeta^t$-inexactness, as defined in Definition~\ref{defn1}. By adopting the zeroth-order inexactness, we relax the exactness requirement of the solution, allowing an error margin around the optimal objective \emph{function} values, unlike the first-order notion of inexactness, which is based on the gradients. Unlike Section~\ref{sec:convergence_analysis}, which explicitly models data heterogeneity through the $(B,H)$-LGD condition, this section deliberately avoids any gradient similarity assumptions. Instead, we analyze RoCo-FL under minimal boundedness conditions that capture worst-case client behavior arising from extreme system noise, fading-induced partial participation, and truncated local computation. Consequently, the results in this section are not tighter in rate, but are more broadly applicable when heterogeneity cannot be meaningfully parameterized. We use the following assumptions for all clients:
%This approach enables a robust and flexible analysis, where we completely avoid using the $(B,H)$-LGD assumption.
\begin{assum}
    The local functions $f_k(\cdot):\mathbb{R}^d \rightarrow \mathbb{R}$ are non-convex, $L$-Lipschitz smooth.
\end{assum}
\begin{assum}
    {The local functions $f_k(\cdot):\mathbb{R}^d \rightarrow \mathbb{R}$ are $G$-Lipschitz continuous if $\vert f_k(\bmtheta)-f_k(\bmtheta')\vert \leq G\Vert\bmtheta-\bmtheta'\Vert$ for all $\bmtheta, \bmtheta' \in \mathbb{R}^d$. }
\end{assum}
\noindent In particular, for a function $f_k(\bmtheta)$ to be $G$-Lipschitz, it satisfies $\Vert \nabla f_k(\bmtheta)\Vert\leq G, \forall \bmtheta$ which is similar to the commonly used bounded gradient assumption~\cite{mcmahan2017communication,FedPD}. { It is well-understood that the $G$-Lipschitz assumption shares similarities with the $(0,H)$-LGD assumption. However, the $G$-Lipschitz assumption does not relate local and global gradients, and hence, it does not explicitly model statistical heterogeneity, enabling convergence analysis under unstructured heterogeneity and system uncertainty~\cite{yuan2022convergence_newfedprox}.} 

\noindent In the following lemma, we derive an upper bound on the precoding factor $p^t$ under zeroth-order inexactness, which is used in the subsequent results.
\begin{lemma}
The precoding factor $p^t$ in each communication round can be upper bounded as $\frac{1}{p^t}\leq \frac{2G^2(\nu^t)^2}{P}$ under zeroth-order inexactness for $K^2 \tau>2$.
\label{lem:precodingfactor_zerothorder}
\end{lemma}
\noindent We now present the result that establishes convergence of the RoCo-FL framework for the most general case of partial device participation. 

\begin{theorem}
\label{G_theorem}
    Given the zeroth-order $\zeta^t$-inexactness in Definition~\ref{defn1}, Assumptions~4-5, $\nu^t=\frac{1}{\lambda} = \frac{1}{3L}\sqrt{\frac{\hat K}{T}}$ and $\zeta^t\leq \min\bigg\{\frac{G}{2L\sqrt{\hat K}},\frac{G}{2L},\frac{G\nu^t}{\hat K}\bigg\}$, we have
\begin{small}
\begin{align}
&\frac{1}{T}\sum_{t=0}^{T-1}\Ex\left[\norm*{\nabla F(\tilde\bmtheta^{t})}^2\right] 
\leq\frac{L\Delta+\tilde C}{\sqrt{T\hat K}},
\label{eq:G_theorem_eq}
\end{align}
\end{small}
where $\tilde C= G^2\left(1+\frac{ d\sigma^2}{3 L P \hat K}(3 L\nu^2+L^2\nu+1+L)\right)$.
\end{theorem}
\noindent Theorem~\ref{G_theorem} shows that RoCo-FL achieves a convergence rate of 
$\mathcal{O}(1/\sqrt{\hat K T})$ for non-convex federated optimization under zeroth-order inexactness. 
This rate is comparable to that of standard stochastic gradient-based FL methods under partial participation, while explicitly accounting for inexact local computation and system-induced perturbations. 
The dependence on $\hat K$ highlights the benefit of client participation, and the result holds without requiring restrictive local gradient similarity assumptions, thereby extending existing convergence guarantees to more realistic FL settings. {Convergence guarantees under similar assumptions have been studied in non-wireless federated settings~\cite{yuan2022convergence_newfedprox}. To the best of our knowledge, this is among the first results to extend such guarantees to non-convex federated optimization with inexact local updates and noisy OTA aggregation.}
\\
 We now extend the convergence analysis of RoCo-FL under $\zeta^t$-inexactness in the presence of fading.
\begin{corollary}Under the conditions stated in Theorem~\ref{G_theorem} and $\hat r>0$, the bound in the presence of fading is given by
\begin{small}
\begin{align}
&\frac{1}{T}\sum_{t=0}^{T-1}\Ex\left[\norm*{\nabla F(\tilde\bmtheta^{t})}^2\right] 
\leq \frac{ L\Delta+\tilde C_f}{\sqrt{T\hat K}}
\label{eq:fading_G_M}
\end{align}
\end{small}
where $\tilde C_f= G^2\left(1+\frac{ d\sigma^2}{3 \hat r^2 L P \hat K}(3 L\nu^2+L^2\nu+1+L)\right)$.
\end{corollary}
\noindent The additional dependence on $\hat r$ in \eqref{eq:fading_G_M} highlights the impact of channel fading. The effect scales as $\mathcal{O}(1/\hat r^2)$, indicating that severe fading (small $\hat r$) amplifies the noise effect, leading to slower convergence. Conversely, a higher $\hat r$ reduces the number of participating clients per round, which also degrades convergence. However, as observed in \eqref{eq:fading_G_M}, our method can mitigate the impact of severe fading by increasing the SNR, while the effect of low participation can be counteracted by longer training duration $T$, intuitively requiring more rounds to achieve convergence.

\section{Experimental Results}
\label{sec:experiments}
\noindent In this section, we present experimental evaluations of the RoCo-FL algorithm across three benchmark datasets, comparing its performance against established OTA-FL baselines such as COTAF \cite{sery2021over}, RobustComm \cite{ang2020robust}, FedProx \cite{li2020federated}, and their respective variants. We utilize various deep neural network architectures for image classification tasks on datasets such as CIFAR10, CIFAR100~\cite{krizhevsky2009learningcifar}, and FEMNIST~\cite{caldas2019leafbenchmarkfederatedsettings}, representing realistic scenarios characterized by non-convex optimization landscapes. Subsequently, we investigate the robustness of the proposed method under two challenges encountered in federated settings: (a) statistical and system heterogeneity among clients, and (b) impairments arising from wireless communication channels, including AWGN and fading.

\noindent \textbf{Datasets:}~The FEMNIST dataset, a federated variant of the MNIST dataset, consists of a large collection of handwritten digit images, each of size $28\times 28$ pixels, categorized into $47$ classes. The CIFAR10 dataset, a widely-used benchmark in computer vision, contains images from $10$ different categories, each of size $32 \times 32$ pixels and consisting of three channels (RGB).  The CIFAR100 dataset comprises of 100 object classes with 600 images per class, making it a more complex and fine-grained classification task.

\noindent For our image classification task on the FEMNIST and CIFAR10 datasets, we employ a Convolutional Neural Network (CNN) that comprises of two convolutional layers followed by two fully connected layers for FEMNIST and three convolutional layers followed by two fully connected layers for CIFAR10. For the CIFAR100 dataset, we employ a ResNet-18 classifier to leverage deeper feature extraction and to improve performance on a more challenging dataset.  All datasets are partitioned among clients in both IID and non-IID ways.

 \noindent \textbf{Baselines:}~To demonstrate the reliability of our proposed method in noisy scenarios, we compare its performance against several baselines. In particular, the baselines we consider include COTAF~\cite{sery2021over}, which applies precoding at the clients to mitigate channel noise while implementing FedAvg, and RobustComm~\cite{ang2020robust}, a robust FL technique designed for noisy wireless environments. We also consider NoisyFedAvg and NoisyProx for the ablation study, not as competitive baselines, which are the implementations of FedAvg and FedProx, respectively, without any means of mitigating the noise effects.
 The relationships between these baselines are as follows: (a)~RoCo-FL with $\sigma=0$ and $p^t=1$ reduces to FedProx\cite{li2020federated}, (b) RoCo-FL with $\lambda=0$ and $0$\% stragglers is equivalent to COTAF\cite{sery2021over}, (c) RoCo-FL with $p^t=1$ corresponds to NoisyProx, (d) COTAF with $\sigma=0$ and $p^t=1$ is equivalent to FedAvg, (e) COTAF with $p^t=1$ is the same as NoisyFedAvg and (f) RoCo-FL with $\lambda=\sigma^2$ aligns with RobustComm\cite{ang2020robust}. Through these comparisons, we highlight the robustness and efficiency of our proposed technique.
 
\begin{figure}[htbp]
\includegraphics[width=\linewidth]{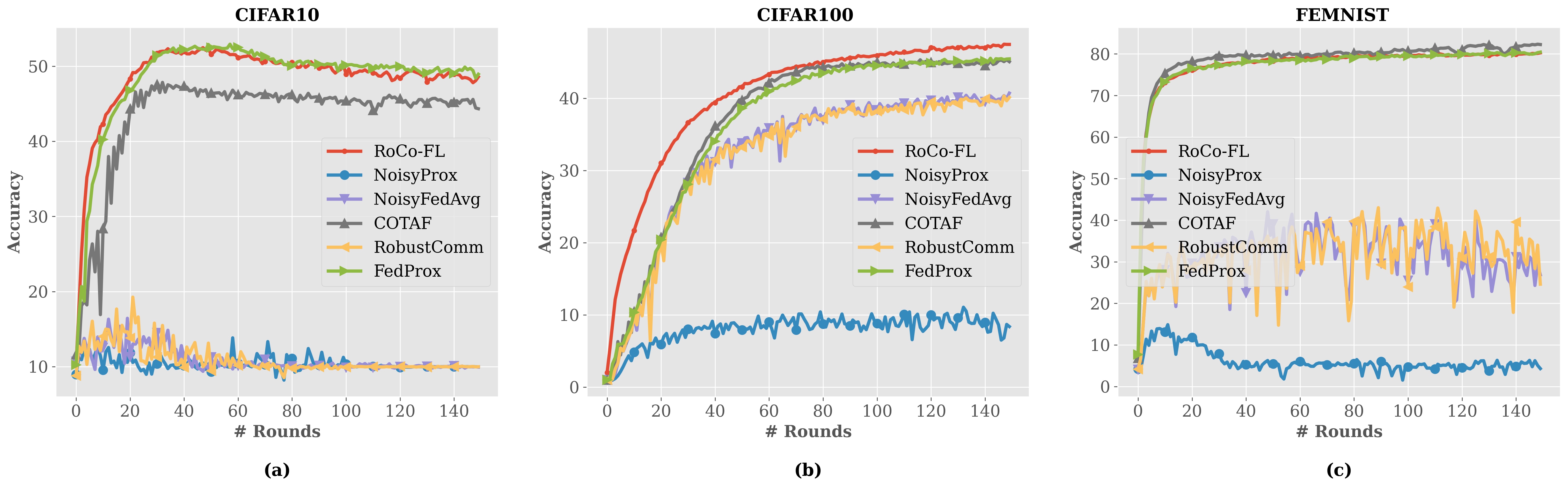}
      
    \caption{Comparison of RoCo-FL with the baselines for FEMNIST and CIFAR10 $(\pi = 0.1, \lambda=0.4)$ and CIFAR100 $(\pi =  0.5, \lambda=0.01)$ for SNR($\tau$) = $0$ dB.}
     \label{fig:all_snr0_sim0Pt1}
\end{figure}

\begin{figure}[htbp]
 \includegraphics[width=\linewidth]{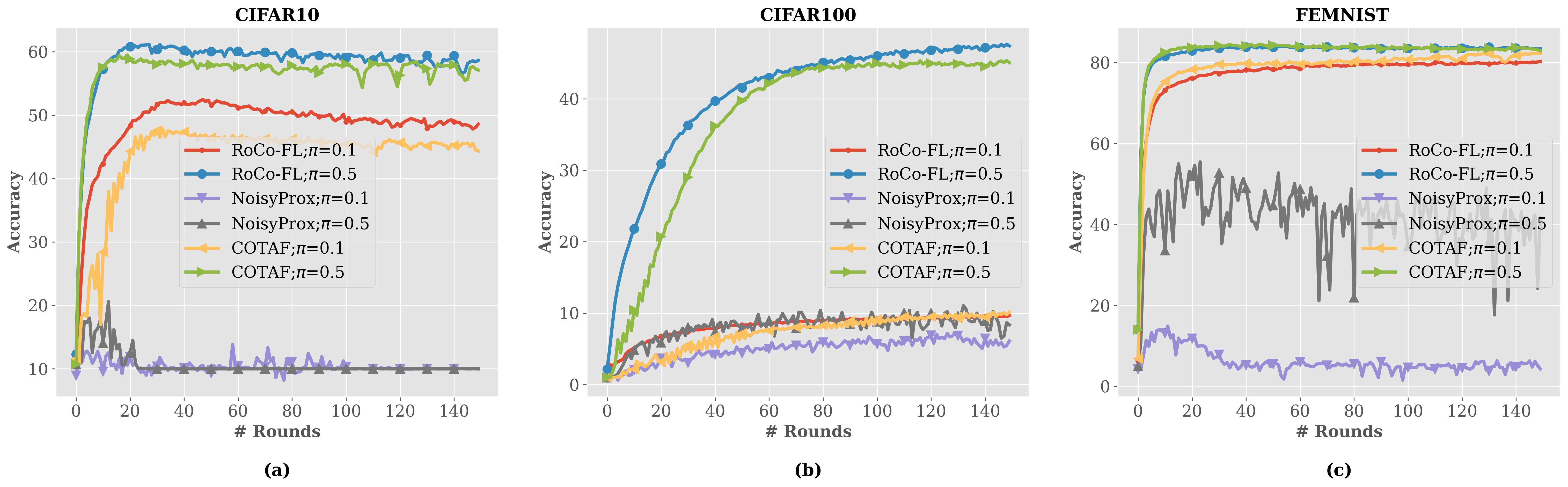}

    \caption{Performance of RoCo-FL and baselines for varying data heterogeneity for SNR = $0$dB.}
    \label{fig:similarity}
    % \vspace{-2mm}
\end{figure}

 \noindent \textbf{Experimental Settings:} The default experimental settings are as follows: the number of clients is set to $K = 30$, the mini-batch size is $64$ for all datasets, and the local epochs are fixed at $E = 3$ at all clients in scenarios without system heterogeneity. The SNR of the MAC channel is set to $0$dB unless stated otherwise. $\lambda$ is a hyperparameter and is set to a value of $0.4$ for FEMNIST and CIFAR10 datasets and $0.01$ for CIFAR100. We explore distinct data heterogeneity scenarios to evaluate the performance of our approach. In the IID scenario, training data is distributed among users such that each user possesses an equal distribution of images across all classes. In the non-IID scenario, ($(1-\pi)*100 \%$) of each user's training data is associated with a single label, leading to varied data distributions among users. The value of $\pi$ lies between $0$ and $1$, where the higher the $\pi$, the higher the similarity of data among clients and the lower the data heterogeneity~\cite{sery2021over}. The default similarity parameter $\pi$ is set to $0.5$ unless specified otherwise. 
 
\noindent We assess the models' performance in terms of classification accuracy and convergence behavior across different data-heterogeneity scenarios, as detailed above. Furthermore, we analyze the influence of FL and wireless channel attributes, including system heterogeneity, client participation, communication rounds, SNR, and fading, on the models' training and performance.

\subsection{Comparison with Baselines}

\noindent In this subsection, we compare the proposed RoCo-FL framework against the baselines described in the previous section, considering the default parameter settings. We observe from Fig.~\ref{fig:all_snr0_sim0Pt1}, that the proposed method (RoCo) performs on par with FedProx while exhibiting robustness to both heterogeneity and wireless channel noise. This indicates that the precoding strategy in RoCo-FL is successful in mitigating the effect of noise. To emphasize this further, we have also included the convergence behavior of NoisyFedAvg and NoisyProx, which do not incorporate any precoding. We observe that both NoisyFedAvg and NoisyProx fail to converge in most cases. Further, RoCo-FL surpasses COTAF in identical wireless settings, highlighting its superior performance, benefiting from the proposed proximal term. Additionally, RoCo-FL consistently outperforms RobustComm, where RobustComm incorporates a constant noise-variance-based Lagrange multiplier for the proximal term. This signifies that our proposed strategy to search for the optimal Lagrange multiplier ($\lambda$) for the proximal term is better suited than \cite{ang2020robust}. Another important observation is that RoCo-FL demonstrates superior stability, i.e., fewer accuracy variations across communication rounds and faster convergence on all the datasets as compared to other methods.

\subsection{Ablation Study}
\noindent We study the effect of varying data and system heterogeneity, SNR, and fading on the proposed algorithm as compared to the baselines. We also substantiate some of the observations made in the convergence analysis given in Sec .~\ref{sec:convergence_analysis} and Sec .~\ref{sec:zeta_convergence}.

\noindent \textbf{Varying Data Heterogeneity ($\pi$)}:~ We demonstrate the performance of the proposed approach and the baselines by varying $\pi$, which is the similarity parameter as described earlier.  As depicted in Fig.~\ref{fig:similarity}, we observe that RoCo-FL exhibits only a slight decrease in performance with an increase in data heterogeneity (from $\pi = 0.5$ to $\pi = 0.1$) on FEMNIST and CIFAR10  datasets, while still maintaining high accuracy and stability, signifying its resilience to varying data distributions across clients. For CIFAR100, the degradation is more pronounced due to the high number of classes compared to other datasets. The performance on the CIFAR10 and  CIFAR100 datasets highlights the effectiveness of the proximal term in mitigating the effects of data heterogeneity, as compared to COTAF. We also observe that NoisyProx barely converges for the above setting. Hence, the robustness and stability demonstrated by RoCo-FL make it an excellent choice for FL scenarios characterized by diverse data distributions in noisy environments.

    \begin{figure}[htbp]
    \centering
      \includegraphics[width=\linewidth]{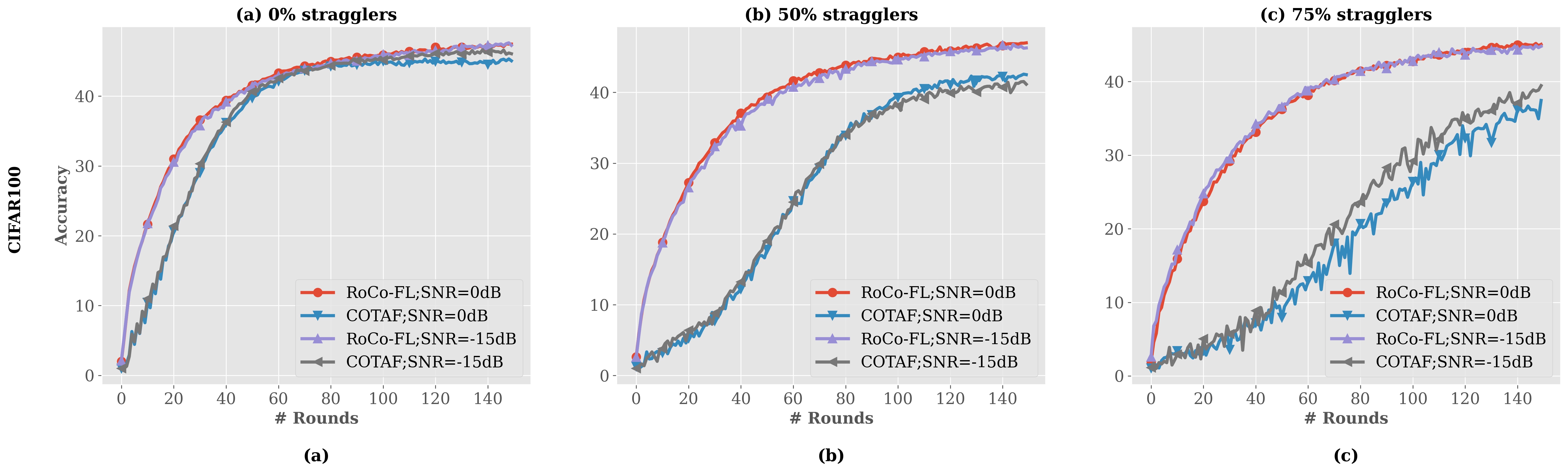}
    \caption{Performance of RoCo-FL and COTAF under heterogeneous system settings on CIFAR100 dataset: (a) $0\%$ stragglers, (b) $50\%$ stragglers, and (c) $75\%$ stragglers.}
    \label{fig:Stragglers_cifar100}
    % \vspace{-5mm}
    \end{figure}
    % \vspace{1em}

    \begin{figure}[htbp]
    \centering
      \includegraphics[width=\linewidth]{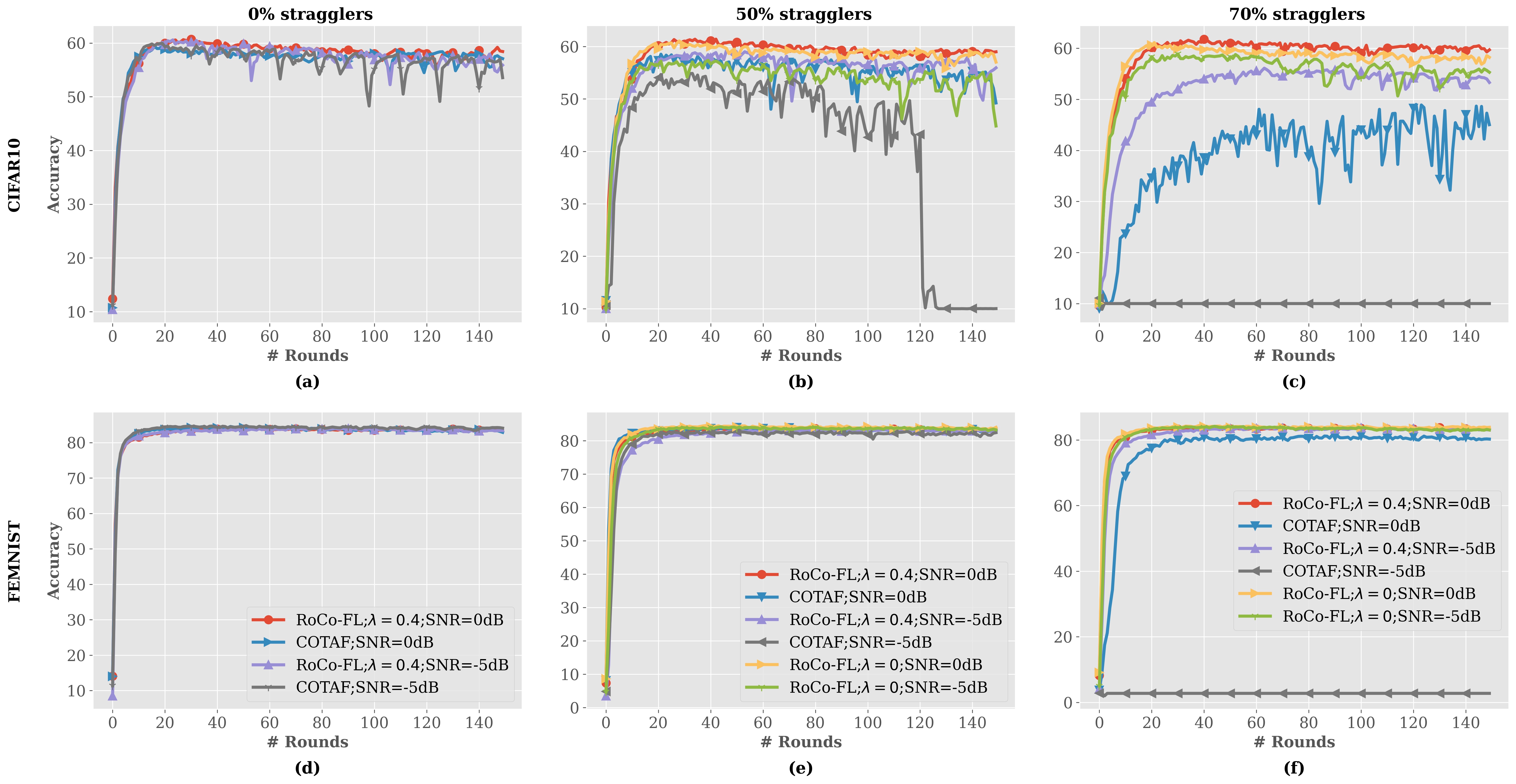}
    \caption{Performance of RoCo-FL and COTAF under heterogeneous system settings on CIFAR10 and FEMNIST datasets: (a),(d) $0\%$ stragglers, (b),(e) $50\%$ stragglers, and (c),(f) $75\%$ stragglers.}
    \label{fig:Stragglers}
    \end{figure}

\begin{figure}[htbp]
 \centering
     \includegraphics[width=\linewidth]{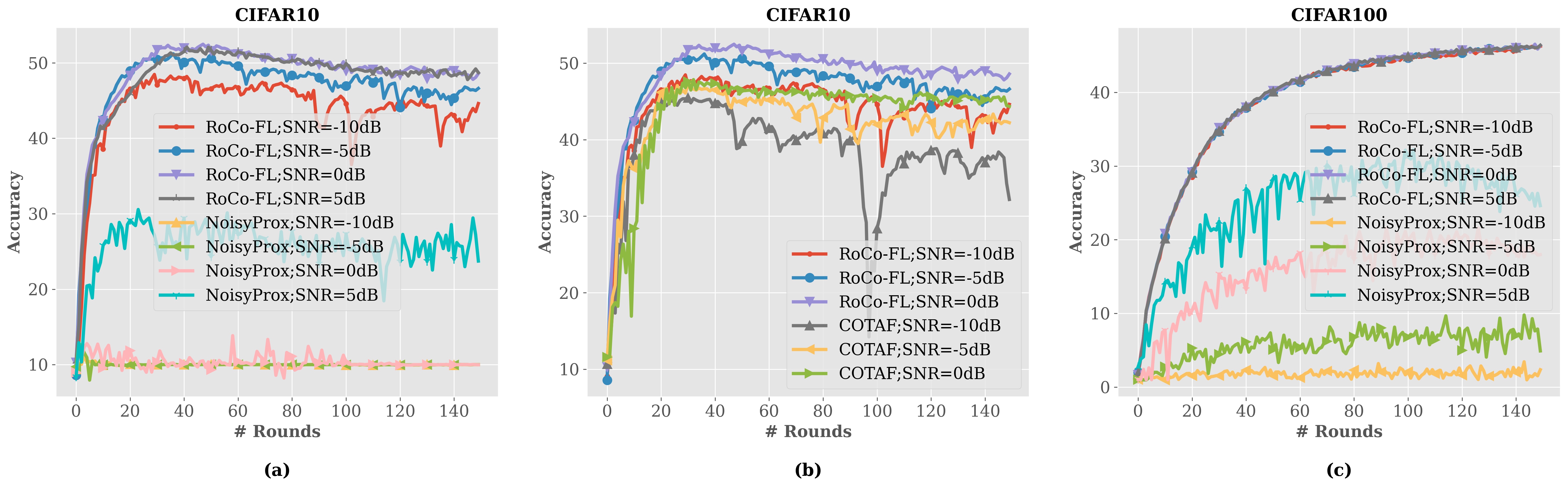}
          
    \caption{Performance of RoCo-FL, NoisyProx, and COTAF for varying SNR on CIFAR10 $(\pi= 0.1)$ and CIFAR100 $(\pi= 0.5)$ datasets.}
    \label{fig:snr_plot1}
\end{figure}

 \begin{figure}[htbp]
 \centering
     
         \includegraphics[width=\linewidth]{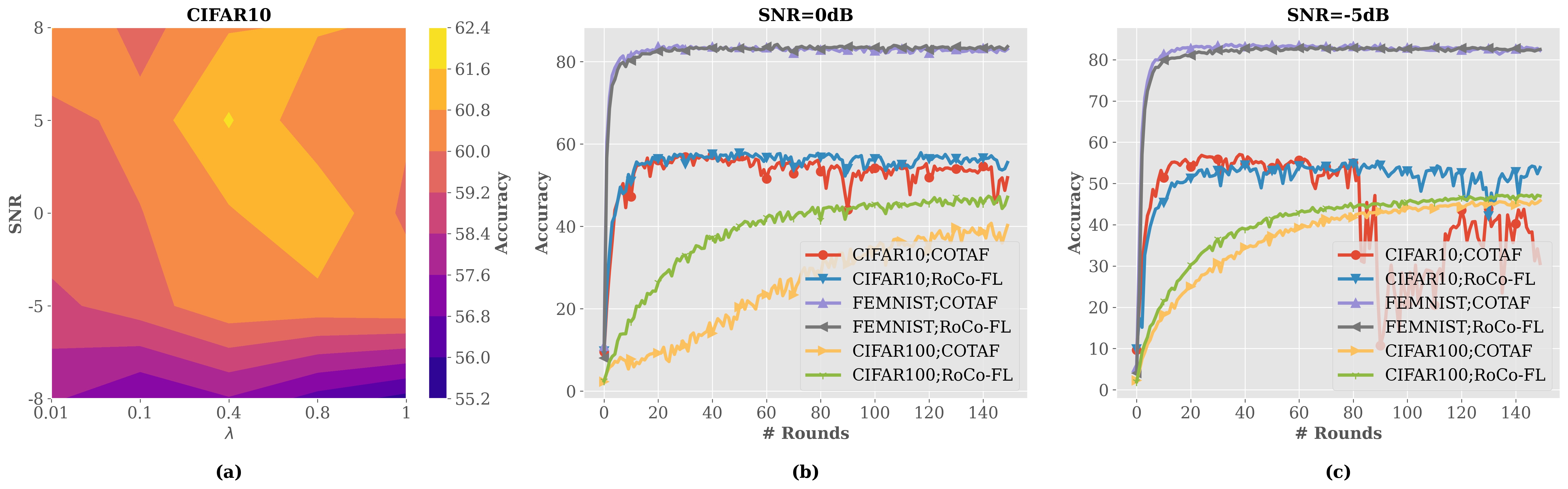}

    \caption{(a) Contour plot for SNR Vs $\lambda$ on CIFAR10 dataset $(\pi=0.5)$. (b) Performance of RoCo-FL and COTAF under fading channels at an SNR of $0$ dB, where $\hat r$ is selected such that the minimum number of participating clients is $10$. (c) Performance of RoCo-FL and COTAF under fading channels at an SNR of $-5$ dB, where $\hat r$ is selected such that the minimum number of participating clients is $20$.}
    \label{fig:contour_fading}
\end{figure}

 \begin{figure}[htbp]
    \centering
      
       \includegraphics[width=\linewidth]{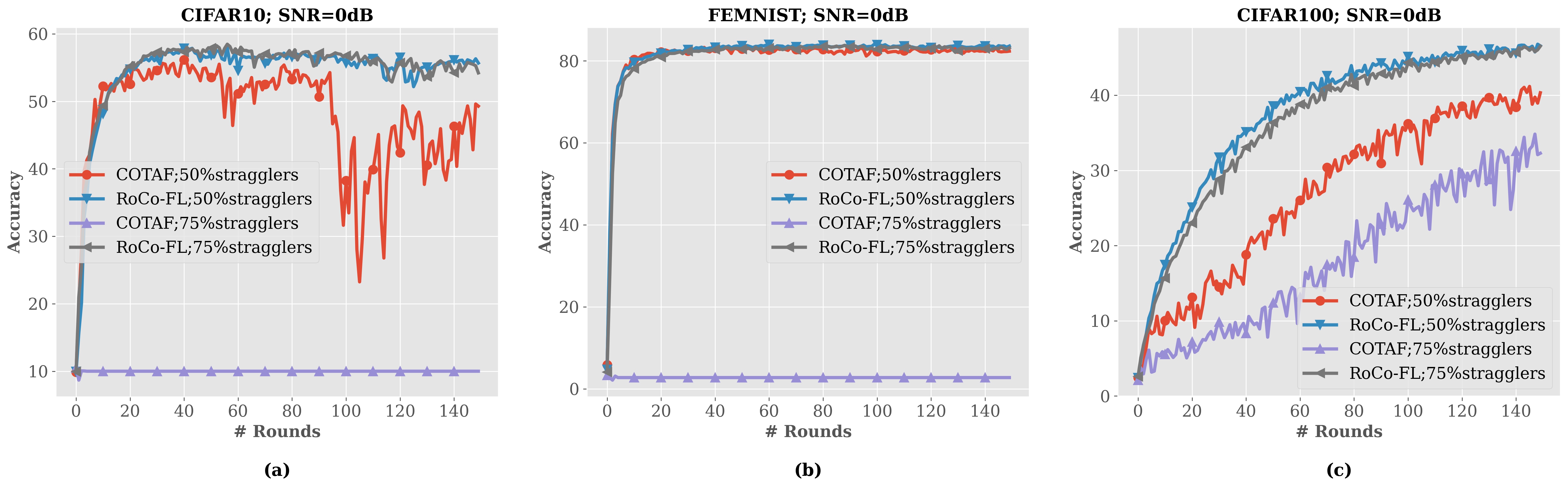}
         
    \caption{Performance of RoCo-FL and COTAF under the effect of fading and stragglers with the SNR of $0$dB. $\hat r$ is chosen such that the minimum number of participating clients is $20$.}
    \label{fig:fading_straggler}
    \end{figure}

    \noindent \textbf{Varying System Heterogeneity:}
 We conduct extensive experiments to evaluate the effects of varying levels of system heterogeneity by introducing stragglers into the FL system. We introduce stragglers which are assumed to have varying system capabilities. We examine the proposed framework and the baselines by introducing different percentages of stragglers: $0\%, 50\%,$ and $75\%$. The plots in Fig.~\ref{fig:Stragglers_cifar100} and Fig.~\ref{fig:Stragglers} depict the performance of RoCo-FL and COTAF under heterogeneous system settings on the CIFAR10, CIFAR100, and FEMNIST datasets. Each column of plots corresponds to a different percentage of stragglers: $0\%$, $50\%$, and $75\%$.
 RoCo-FL is designed to handle stragglers, whereas the COTAF addresses stragglers in each round by dropping them. It is important to note that in the case of $0\%$ stragglers, RoCo-FL with $\lambda=0$ is equivalent to COTAF as depicted in Fig.~\ref{fig:Stragglers}. In the scenario with $50\%$ stragglers, we observe that COTAF's performance deteriorates when the SNR is -$5$dB in Fig.~\ref{fig:Stragglers}. This deterioration becomes more pronounced in the case of $75\%$ stragglers, where COTAF is unable to perform well, even on the FEMNIST dataset. Conversely, RoCo-FL maintains robust performance even as the proportion of stragglers increases. This trend is also observed in the case of CIFAR100 in Fig.~\ref{fig:Stragglers_cifar100}.

 \noindent \textbf{Varying SNR:}
We demonstrate the performance of the proposed approach and the baselines by varying SNR, i.e., we verify the robustness of the proposed approach in the presence of detrimental effects of noise at the server. From Fig.~\ref{fig:snr_plot1}, it is evident that RoCo-FL demonstrates robustness to varying SNR. The performance is shown for SNR  ranging from $-10$dB to $5$dB. The proposed method effectively mitigates the impact of noise, with only a slight decrease in accuracy observed at SNR of $-10$dB on FEMNIST, CIFAR10. The robustness to varying SNR is highly evident in the CIFAR-100 dataset. Additionally, the effectiveness of the proximal term is highlighted on the CIFAR10 and CIFAR100 datasets, where we compare RoCo-FL with NoisyProx and COTAF across varying SNR. These observations clearly indicate that RoCo-FL exhibits greater robustness as compared to COTAF. It is also observed that when SNR increases, the proposed method performs better, which indicates that small levels of noise are fully mitigated in RoCo-FL. As discussed in the convergence analysis in the theorems of Sec .~\ref{sec:convergence_analysis} and \ref{sec:zeta_convergence}, it is experimentally illustrated in Fig.~\ref{fig:snr_plot1}, that lower values of SNR slow down the convergence rate.

\noindent In Fig.~\ref{fig:contour_fading} (a), we present a contour plot depicting the relationship between SNR and $\lambda$, and its effect on the resulting accuracy of the proposed RoCo-FL method on the CIFAR10 dataset. The plot shows that accuracy generally increases with higher SNR values, as expected. As SNR moves from $-8$dB to $8$dB, the accuracy improves, indicating that RoCo-FL is more effective in less noisy environments. It is evident that the parameter $\lambda$ significantly influences the accuracy. For small values of $\lambda$ (close to $0.01$), accuracy is generally lower on the CIFAR10 dataset. As $\lambda$ increases, accuracy initially improves, peaking at around $\lambda \approx 0.4$, where the highest accuracy is observed. For higher values of $\lambda$, accuracy does not improve significantly and may also lead to a slight decline. This was also observed in the theoretical analysis in Sec .~\ref{sec:convergence_analysis}, where we concluded that large values of $\lambda$ hindered convergence. This indicates that tuning $\lambda$ appropriately can significantly enhance the robustness and accuracy of RoCo-FL, especially in noisy environments. We have observed that values of $\lambda > 1$ result in poor accuracy performance. This is attributed to the fact that a larger $\lambda$ gives more emphasis on the proximal term rather than learning from local data.

\vspace{1em}
\noindent \textbf{Effect of Fading:} We experimentally analyze the impact of fading on the proposed setup. In Fig.~\ref{fig:contour_fading} (b), (c) and  Fig.\ref{fig:fading_straggler}, we illustrate the performance of RoCo-FL and COTAF in the presence of fading on the CIFAR10, CIFAR100, and FEMNIST datasets. We observe that the proposed method is successful in handling the wireless channel fading effects.
Further, we also investigate the effect of fading under varying SNR and in the presence of stragglers. Fig.~\ref{fig:contour_fading} shows that RoCo-FL achieves higher accuracy as compared to COTAF under fading conditions, especially in CIFAR100 and CIFAR10 datasets. Fig.~\ref{fig:fading_straggler} demonstrates the robust performance of RoCo-FL under the joint effect of fading and stragglers on the CIFAR10 and CIFAR100 datasets, respectively. In contrast, COTAF fails to converge in scenarios with $75\%$ stragglers on all the datasets.

% \vspace{2em}

\section{Conclusions}
\label{sec:conc}
\noindent In this work, we proposed the RoCo-FL framework, which is a robust OTA-FL algorithm that tackles fading, noise uncertainty, and data and system heterogeneity that are inherent to wireless federated networks. We derived novel constrained local optimization problems that incorporated the effects of noise and fading at each client. The constraint, also called the proximal term, played a dual role where it allowed for variable amounts of work to be performed locally across clients and proved to be effective at mitigating noise and fading in each communication round. RoCo-FL uses precoding for long-term noise robustness. We provide convergence guarantees for RoCo-FL for locally non-convex optimization problems, while solving for first-order inexact solutions in mild and severe heterogeneity settings, as characterized by the local gradient dissimilarity assumption. {We also provide convergence guarantees for RoCo-FL that seek zeroth-order inexact solutions under heterogeneity-agnostic boundedness assumptions characterized by Lipschitz continuity.} Our empirical evaluation across a suite of federated datasets has validated our theoretical analysis and demonstrated that the RoCo-FL framework can significantly stabilize and improve the convergence behavior of FL in realistic wireless heterogeneous networks.

\noindent \textbf{Limitations and Future Directions.}
RoCo-FL establishes a comprehensive, full-fledged wireless FL framework that jointly handles wireless channel impairments, data, and system heterogeneity within a unified theoretical foundation, backed by extensive experimental results. Nonetheless, several directions remain open. RoCo-FL assumes perfect CSI at the 
clients for precoding and threshold-based 
selection, extending the analysis to 
imperfect or estimated CSI is an important 
practical direction. 
The framework further assumes synchronous aggregation and a single-antenna setup, 
and extensions to asynchronous protocols and massive MIMO settings with beamforming 
would better reflect practical deployments. The proximal parameter $\lambda$ requires 
tuning, and deriving a fully automatic channel-aware adaptive schedule with provable feasibility guarantees at every round remains open. Finally, the structural similarity between OTA aggregation noise and Gaussian perturbation mechanisms suggests a natural connection to differentially private FL, where the channel itself 
provides privacy without additional noise injection, and formalizing this privacy-utility trade-off is a compelling direction for future work.

\section{Appendix}
\subsection{Proof sketch under $(B,H)$-LGD Assumption}
\begin{lemma}
\label{lem:boundingthegradients}
(Bounding the gradients) Let assumption~3 hold. Then, for any $\bmtheta, {\tilde{\bmtheta}} \in \mathbb{R}^d$ and $\vecw^t = {\tilde\bmtheta}^t-{\bmtheta}^t$ we have the following bound on the norm of the gradient of $f(\cdot)$: $\norm{\nabla{f({\bmtheta}^t)}}\leq \norm{\nabla{f({\tilde{\bmtheta}}^t)}}+ L\norm{\vecw^t}.$
% \begin{align}
%   \norm{\nabla{f({\bmtheta}^t)}}\leq \norm{\nabla{f({\tilde{\bmtheta}}^t)}}+ L\norm{\vecw^t}.
% \end{align}
% \label{lem:boundingthegradients}
\end{lemma}

\begin{lemma}[Modified $\gamma^t$-inexactness]
\label{lem:norm_gradient_h}
Given the first-order inexactness condition with $\gamma^t=\max_{k} \gamma_k^t$, i.e.,
$\|\nabla h_k(\bmtheta_k^{t+1};\bmtheta^t)\|\le \gamma^t\|\nabla f_k(\bmtheta^t)\|$,
the corresponding bound for noisy aggregates $\tilde{\bmtheta}^t=\bmtheta^t+\vecw^t$
satisfies
\[
\|\nabla h_k(\bmtheta_k^{t+1};\tilde{\bmtheta}^t)\|
\le
\gamma^t\|\nabla f_k(\tilde{\bmtheta}^t)\|
+
(\gamma^t L+\lambda)\|\vecw^t\|.
\]
\end{lemma}
\begin{proof}
From Definition~1,
$h_k(\bmtheta;\tilde{\bmtheta}^t)
=f_k(\bmtheta)+\frac{\lambda}{2}\|\bmtheta-\tilde{\bmtheta}^t\|^2$,
and hence
$\nabla h_k(\bmtheta_k^{t+1};\tilde{\bmtheta}^t)
=\nabla f_k(\bmtheta_k^{t+1})+\lambda(\bmtheta_k^{t+1}-\bmtheta^t)-\lambda\vecw^t$.
By the triangle inequality,
\[
\|\nabla h_k(\bmtheta_k^{t+1};\tilde{\bmtheta}^t)\|
\le
\|\nabla h_k(\bmtheta_k^{t+1};\bmtheta^t)\|+\lambda\|\vecw^t\|.
\]
Applying the $\gamma^t$-inexactness condition and Lipschitz smoothness
$\|\nabla f_k(\bmtheta^t)-\nabla f_k(\tilde{\bmtheta}^t)\|\le L\|\vecw^t\|$
yields the result.
\end{proof}

\noindent Under Definition~1, suppose $\nabla^2 f_k(\bmtheta)\succeq -\bar L \mathbf I$ and let
$\bar\mu \triangleq \lambda-\bar L>0$.
Then the local objective $h_k(\bmtheta;\tilde{\bmtheta}^t)
=f_k(\bmtheta)+\tfrac{\lambda}{2}\|\bmtheta-\tilde{\bmtheta}^t\|^2$
is $\bar\mu$-strongly convex for all $t$, since
$\nabla^2 h_k(\bmtheta;\tilde{\bmtheta}^t)
=\nabla^2 f_k(\bmtheta)+\lambda\mathbf I\succeq\bar\mu\mathbf I$.

\noindent Let $\bmtheta_k^{*,t+1}$ denote the exact minimizer of $h_k(\cdot;\bmtheta^t)$.
By using strong convexity,$\|\bmtheta_k^{*,t+1}-\tilde{\bmtheta}^t\|
\le\tfrac{1}{\bar\mu}\|\nabla f_k(\tilde{\bmtheta}^t)\|$, first-order inexactness condition, $\|\nabla h_k(\bmtheta_k^{t+1};\bmtheta^t)\|\le\gamma^t\|\nabla f_k(\bmtheta^t)\|$ and triangular inequality, we obtain
\begin{equation}
\label{eq:local_deviation}
\|\bmtheta_k^{t+1}-\tilde{\bmtheta}^t\|
\le
\tfrac{1+\gamma^t}{\bar\mu}\|\nabla f_k(\tilde{\bmtheta}^t)\|
+
\tfrac{\gamma^t L+\lambda}{\bar\mu}\|\vecw^t\|.
\end{equation}

\subsubsection{Proof of Lemmas~\ref{lem:precodingfactor_1storder} and~\ref{lem:precodingfactor_zerothorder}}
From \eqref{eq:precoder1},
$\frac{1}{p^t}\le \frac{1}{P}\sum_{k=1}^K q_k
\mathbb E\|\bmtheta_k^{t+1}-\tilde{\bmtheta}^{t}\|^2$, we get the following
\begin{itemize}
  \item \emph{First-order inexactness:}
Using \eqref{eq:local_deviation} and Assumption~1, we get
\[
\frac{1}{p^t}
\le
\frac{2aK^2(B^2 \norm{\nabla F(\tilde\bmtheta)}^2 +H^2)}{K^2P-2bd \sigma^2},
\]
which holds for $\lambda\leq\frac{K^2\tau\bar\mu}{2}$. Here $a=(1+\gamma^t)/\bar\mu$, $b=(\gamma^t L+\lambda)/\bar\mu$.

\item \emph{Zeroth-order inexactness:}
Using
$\|\bmtheta_k^t-\tilde{\bmtheta}^{t-1}\|
\le \nu^t(2L\zeta^t-G)+\|\vecw^t\|$,
$\zeta^t\le G/L$, and $K^2\tau>2$, we obtain $\frac{1}{p^t}\le \frac{2G^2(\nu^t)^2}{P}$.
\end{itemize}
\hfill$\square$

\noindent Let $\bar{\bmtheta}^{t+1}=\mathbb E_k[\bmtheta_k^{t+1}]$.
Using Jensen’s inequality, triangular inequality, \eqref{eq:local_deviation} and invoking the $(B,H)$-LGD assumption yields
\begin{small}
\begin{equation}
\label{eq:avg_deviation}
\|\bar{\bmtheta}^{t+1}-\bmtheta^t\|
\le
\tfrac{1+\gamma^t}{\bar\mu}\big(B\|\nabla F(\tilde{\bmtheta}^t)\|+H\big)
+
\tfrac{\bar\mu+\gamma^t L+\lambda}{\bar\mu}\|\vecw^t\|.
\end{equation}
\end{small}

\noindent Since $\Ex_k\left[\nabla f_k(\tilde{\bmtheta}^t)\right] = \nabla F(\tilde{\bmtheta}^t)$, we simplify $\bar \bmtheta^{t+1} - {\bmtheta}^t = $ as 
\begin{small}
\begin{align}
   \Ex_k[{\bmtheta}_k^{t+1}] - {\bmtheta}^t = \frac{-1}{\lambda}\left(
\nabla{F(\tilde{\bmtheta}^t)}+\vecm^{t+1}\right)+\vecw^t.
\label{eq:expmt+1}
\end{align}
\end{small}
where $\vecm^{t+1}\triangleq \Ex_k\left[\nabla{f_k({\bmtheta}_k^{t+1})}-\nabla{f_k(\tilde{\bmtheta}^t)-\nabla{h_k(\bmtheta_k^{t+1};{\tilde{\bmtheta}}^t)}}\right]$.
\noindent We use the $L$-Lipschitz smoothness assumption, triangle inequality, \eqref{eq:local_deviation}, Lemma~\ref{lem:norm_gradient_h} and Assumption~1 to obtain an upper bound on $\norm{\vecm^{t+1}}$.
Using $L$-smoothness of $F(\cdot)$, bounding cross terms via Cauchy–Schwarz
and above abounds, and taking expectation over OTA noise yields
\begin{equation}
\label{eq:descent_gamma}
\mathbb E[F(\bar{\bmtheta}^{t+1})]
\le
F(\tilde{\bmtheta}^t)
-\alpha\|\nabla F(\tilde{\bmtheta}^t)\|^2
+\beta H^2, 
\end{equation}
where $\alpha>0$ for admissible $\lambda$.
Now, let $\hat{K} < K$ clients in the set  $\mathcal{S}^t \subset [K]$ participate in the FL process in the $t$-th communication round. Applying the local Lipshitz condition, $F({\tilde\bmtheta}^{t+1}) \leq F(\bar{\bmtheta}^{t+1})+L_0 \norm{\tilde\bmtheta^{t+1}-\bar\bmtheta^{t+1}}$
where $L_0$ is the local Lipschitz constant. After $\Ex_{\mathcal{S}^t}[.]$ on both sides, we get
\begin{align}
    \Ex_{S^t}[F(\tilde\bmtheta^{t+1})] \leq \Ex_{S^t}[F(\bar\bmtheta^{t+1})] +q^t,
    \label{eq:Theorem2secondstep_qt_shortened}
\end{align}
where $q^t=\Ex_{S^t}[L_0 \norm{\tilde\bmtheta^{t+1}-\bar\bmtheta^{t+1}}]$. We follow the proof steps in \cite{li2020federated} to obtain the upper bounds on $q^t$ and $L_0$. Finally, combining \eqref{eq:descent_gamma} and \eqref{eq:Theorem2secondstep_qt_shortened}, taking expectation w.r.t $\hat \vecw^t$, averaging and telescoping, we prove the Theorems~\ref{partial_participation_theorem_BH_M}.
\hfill$\square$
\subsection{Proof Sketch under $G$-Lipschitz Assumption}
\noindent Each client solves
$\bmtheta_k^t=\argmin_{\bmtheta}\Big(f_k(\bmtheta)+\tfrac{1}{2\nu^t}\|\bmtheta-\tilde{\bmtheta}^{t-1}\|^2\Big)$.
We define $\phi_k^t=\nabla f_k(\bmtheta_k^t)$,
$\phi^t=\frac{1}{\hat K}\sum_{k\in\mathcal K^t}\phi_k^t$,
and $\bar\phi^t=\frac{1}{K}\sum_{k=1}^K\phi_k^t$.
\noindent Under $G$-Lipschitzness,
$\mathbb E[\phi^t]=\bar\phi^t$ and
$\mathbb E\|\phi^t-\bar\phi^t\|^2\le G^2/\hat K$.

\noindent With $\zeta^t$-inexactness, smoothness and $\nu^t\le 1/L$,
\[
\|\bmtheta_k^t-\tilde{\bmtheta}^{t-1}\|
\le \nu^t(2L\zeta^t-G)+\|\hat{\vecw}^t\|,
\]
which implies
$
\|\nabla F(\tilde{\bmtheta}^{t-1})-\bar\phi^t\|^2
\le
L^2\!\left(\nu^t(2L\zeta^t-G)+\|\hat{\vecw}^t\|\right)^2$.
Let
$\delta_k^t=\nu^{-1}(\bmtheta_k^t-\tilde{\bmtheta}^{t-1})+\phi_k^t$ and
$\delta^t=\frac{1}{\hat K}\sum_{k\in\mathcal K^t}\delta_k^t$.
Then
$\tilde{\bmtheta}^t
=
\tilde{\bmtheta}^{t-1}
-
\nu^t(\phi^t-\delta^t)
+
\hat{\vecw}^t$ and 
$\max\{\|\delta^t\|,\|\bar\delta^t\|\}
\le 2L\zeta^t+\|\hat{\vecw}^t\|.$

\noindent Finally, using $L$-smoothness of $F$, Young’s inequality, the noise bound
$\mathbb E\|\hat{\vecw}^t\|^2\le d\sigma^2G^2(\nu^t)^2/(\hat K^2P)$ and then telescoping, we prove the theorem.
\hfill$\square$

\subsection{Auxiliary terms of Theorem~$1$} The detailed proof of Theorem~\ref{partial_participation_theorem_BH_M} is deferred to the supplementary.

\begin{small}
\begin{align*}
\hat{\rho}_1&\leq
\frac{1}{\lambda}, \quad
\hat{C}_1=
\Bigg(
 B
 \Bigg[
   L\left(\frac{\gamma^t L + \lambda}{\bar\mu}\right)
   + \gamma^t L + \gamma^t
 \Bigg]
\Bigg)^2
+ \frac{L B^2}{2}+ \frac{3LB^2}{2}
\left(\frac{\bar\mu + \gamma^t L + \lambda}{\bar\mu}\right)^2
+ B^2\Bigg(L+ \frac{\bar\mu+\gamma^t L+\lambda}{\bar\mu}\Bigg)^2\nonumber\\
&
+ B^2
\left(
\frac{\sqrt{2}}{\sqrt{\hat K}}
\left(\frac{\gamma^t L + \lambda}{\bar\mu}\right)
+ 1
\right)^2
+ B^2
\left(1+
\frac{4L(\gamma^t L + \lambda)^2}{\hat K \bar\mu}
\right)^2+ B^2\Bigg(L+ \frac{\bar\mu+\gamma^t L+\lambda}{\bar\mu}\Bigg)\left(
\frac{\sqrt{2}}{\sqrt{\hat K}}
\left(\frac{\gamma^t L + \lambda}{\bar\mu}\right)
+ 1
\right) \nonumber\\
\hat{\rho}_2
&=
\left(\frac{1+\gamma^t}{\bar\mu}\right)^2
\left(
L
+ \frac{3L}{2}
+ 4L^2
+ \frac{2}{\hat K}+\frac{2\sqrt{2}}{\sqrt{\hat K}}
+ \frac{8L}{\hat K}
\right)+ (\gamma^t)^2
+ 2\gamma^t L\left(\frac{1+\gamma^t}{\bar\mu}\right)
\nonumber\\
\hat{C}_2
&=
\Bigg[
L\left(\frac{\gamma^t L + \lambda}{\bar\mu}\right)
+ \gamma^t L + \lambda
\Bigg]^2
+ \frac{L}{2}
+ \frac{3L}{2}
\left(\frac{\bar\mu + \gamma^t L + \lambda}{\bar\mu}\right)^2
+ \left(
L + \frac{\bar\mu + \gamma^t L + \lambda}{\bar\mu}
\right)^2+ \left\{
\frac{\sqrt{2}}{\sqrt{\hat K}}
\left(\frac{\gamma^t L + \lambda}{\bar\mu}\right)
+ 1
\right\}^2\nonumber\\
&+ \left(
L + \frac{\bar\mu + \gamma^t L + \lambda}{\bar\mu}
\right)
\left(
\frac{\sqrt{2}}{\sqrt{\hat K}}
\left(\frac{\gamma^t L + \lambda}{\bar\mu}\right)
+ 1
\right)
+ \left\{
1
+ \frac{4L}{\hat K}
\left(\frac{\gamma^t L + \gamma^t}{\bar\mu}\right)^2
\right\}
\end{align*}
\end{small}

\section*{Declarations}
\subsection*{Availability of data and materials}
The datasets used in this work are publicly available.

\subsection*{Competing interests}
The authors declare that they have no conflict of interest.
\subsection*{Funding}
This research received no external funding.
\subsection*{Author contributions}
Zubair Shaban conceived the idea, developed the theoretical framework, conducted the experiments, and prepared the manuscript. Nazreen Shah contributed to the experimental evaluation and manuscript review. Ranjitha Prasad supervised the work, provided technical guidance, and reviewed the manuscript. All authors read and approved the final manuscript.
\bibliographystyle{IEEEbib}
\bibliography{references}


\begin{table}[htbp]
\centering
\caption{Comparison of assumptions and convergence guarantees of RoCo-FL with existing federated learning methods.
Here, $T$ denotes the total number of communication rounds.}
\label{tab:method_comparison}

\footnotesize
\setlength{\tabcolsep}{3pt}
\renewcommand{\arraystretch}{1.1}
\resizebox{\linewidth}{!}{%
\begin{tabular}{|l|c|c|c|c|c|}
\hline
\textbf{Method} & \textbf{Wireless} & \textbf{Straggler} & \textbf{$(B,H)$} & \textbf{Convergence} & \textbf{Convexity} \\
\textbf{} & \textbf{Aware} & \textbf{Handling} & \textbf{Assump.} & \textbf{Rate} & \textbf{Setting} \\
\hline
FedAvg~\cite{mcmahan2017communication}
& \xmark & \xmark & $(0,H)$ & $\mathcal{O}(1/T)$ & Convex \\
FedProx(A)~\cite{li2020federated}
& \xmark & \cmark & $(B,0)$ & $\mathcal{O}(1/T)$ & Non-Convex \\

FedProx(B)~\cite{yuan2022convergence_newfedprox}
& \xmark & \cmark & \xmark & $\mathcal{O}(1/\sqrt{T})$ & Non-Convex \\

SCAFFOLD~\cite{karimireddy2020scaffold}
& \xmark & \xmark & \xmark & $\mathcal{O}(1/T)$ & Non-Convex \\

FedLin~\cite{mitra2021linear}
& \xmark & \xmark & \xmark & $\mathcal{O}(1/T)$ & Non-Convex \\

FedPD~\cite{FedPD}
& \xmark & \xmark & $(1,H)$ & $\mathcal{O}((1-p)/T)$ & Non-Convex \\

COTAF~\cite{Sery2020Over-the-AirData}
& \cmark & \xmark & $(0,H)$ & $\mathcal{O}(1/T)$ & Convex \\

\textbf{RoCo-FL (Thm.~\ref{partial_participation_theorem_BH_M})}
& \cmark & \cmark & $(B,H)$ & $\mathcal{O}(1/T)+ \mathcal{O}(H^2)$ & Non-Convex \\

\textbf{RoCo-FL (Cor.~\ref{partial_participation_theorem})}
& \cmark & \cmark & $(B,0)$ & $\mathcal{O}(1/T)$ & Non-Convex \\

\textbf{RoCo-FL (Thm.~\ref{G_theorem})}
& \cmark & \cmark & \xmark & $\mathcal{O}(1/\sqrt{T})$ & Non-Convex \\
\hline
\end{tabular}
}
\end{table}

\begin{center}
\section{ \textbf{Supplementary Material for the paper titled "Unified Robust and Convergent Over-The-Air Federated Learning under Channel Noise and Heterogeneity"}}
\end{center}

This supplementary material provides proofs of the results discussed in the manuscript titled "Unified Robust and Convergent Over-The-Air Federated Learning under Channel Noise and Heterogeneity". For ease of reference, the enumeration of all the theorems, lemmas, and equations is consistent with the manuscript. 
% \subsection{Notations}
% Boldface small letters denote vectors, and boldface capital letters denote matrices. $\mathbf{I}_M$ denotes an $M \times M$ identity matrix. The $\ell_2$-norm of a vector $\vecx$ is denoted as $||\vecx||$. $\mathcal{P}$ represents sets and $|\mathcal{P}|$ represents size of the set. We use $\Ex_k[.]$ to represent the expectation taken with respect to all client indices $k=0,..., K$ and $\Ex_{\mathcal{P}}[.]$ denotes the expectation with respect to the chosen set of clients in $\mathcal{P}$. $\nabla f(.)$ denotes gradient of the function $f(.)$. Further, $[M]$ represents the set $\{1,\hdots,M\}$ and $\matA \succcurlyeq \matB$ denotes that $\matA - \matB$ is positive definite.

\subsection{Notations}

\begin{table}[h]
    \centering
    \renewcommand{\arraystretch}{1.3} % Adjust row height
    \begin{tabular}{|c|l|}
        \hline
        \textbf{Notation} & \textbf{Description} \\
        \hline
        $\mathbf{x}$ & Boldface small letters denote vectors \\
        $\mathbf{X}$ & Boldface capital letters denote matrices \\
        $\mathbf{I}_M$ & $M \times M$ identity matrix \\
        $||\bm{x}||$ & $\ell_2$-norm of vector $\bm{x}$ \\
        $\mathcal{P}$ & A set \\
        $|\mathcal{P}|$ & Size of the set $\mathcal{P}$ \\
        $\Ex[\cdot]$ & Expectation with respect all randomnesses including noise\\
        $\mathbb{E}_k[\cdot]$ & Expectation over all client indices $k=0,\dots,K$ \\
        $\mathbb{E}_{\mathcal{P}}[\cdot]$ & Expectation over the chosen set of clients in $\mathcal{P}$ \\
        $\nabla f(\cdot)$ & Gradient of function $f(\cdot)$ \\
        $[M]$ & Set $\{1,\hdots,M\}$ \\
        $\mathbf{A} \succcurlyeq \mathbf{B}$ & $\mathbf{A} - \mathbf{B}$ is positive definite \\
        \hline
    \end{tabular}
    \caption{List of notations used in the paper}
    \label{tab:notations}
\end{table}

\subsection{Additional Experiments}

\subsubsection{Ablation Study}
In this subsection, we study the effect of varying numbers of clients and the precoding factor on the proposed approach.

\noindent \textbf{Effect of varying clients:}

\noindent In this section, we observe the behavior of RoCo-FL for varying clients on CIFAR10 and FEMNIST datasets. We provide the convergence trend for $15,30$ and $45$ clients for $\pi=0.5, \lambda=0.4$ and SNR= $0$ dB. As discussed in the convergence analysis, from the accuracy plot on the CIFAR10 dataset as in Fig.~\ref{fig:clients(cifarfeminist)}, we observe that increasing the number of clients can mitigate the impact of noise in our proposed framework. In contrast, the FEMNIST dataset is less affected by noise, and thus, the accuracy performance on the FEMNIST dataset clearly illustrates the increased communication overhead and higher variance in updates that accompany an increase in the number of clients, ultimately resulting in slower convergence.

\begin{figure}[htbp]
    \centering
      \includegraphics[scale=0.45]{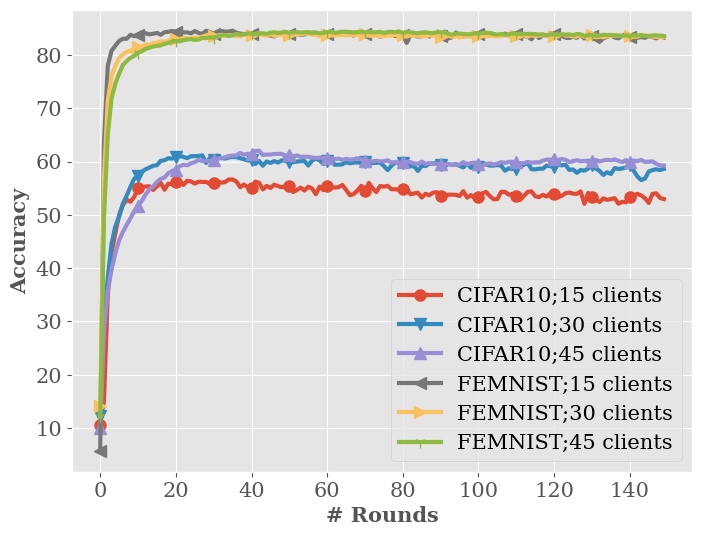}
    \caption{Performance of RoCo-FL under varying clients. }
    \label{fig:clients(cifarfeminist)}
    \end{figure}

\subsection{Detailed Proofs of Key Lemmas and Theorems under $(B,H)$-LGD assumption}
\begin{lemma}
Let there exist ${\bar{L}}>0$ such that $\nabla^2{f}_k(\bmtheta)\succcurlyeq-{\bar{L}}\mathbf{I}$, and let  $\bar{\mu}\coloneqq \lambda - \bar{L} >0 $. Then for any $\bmtheta, {{\bmtheta}^t} \in \mathbb{R}^d$, $h_k(\bmtheta,\bmtheta^t)$, as in Definition~\ref{defn1}, is $\bar{\mu}$-strongly convex for all $t$.
\label{lem:proxconvexity}
\end{lemma}
\noindent \textbf{Proof of Lemma~\ref{lem:proxconvexity}:} Adding $\lambda\mathbf{I}_d$ on the LHS of $\nabla^2{f}_k(\bmtheta) \geq -{\bar{L}}\mathbf{I}_d$, and using the definition of $\bar{\mu}$, we obtain 
\begin{align}
\nabla^2{f}_k(\bmtheta)+\lambda\mathbf{I}_d&\geq\bar{\mu}\mathbf{I}_d.
\end{align}
Using the expression for $h_k(\bmtheta,\bmtheta^t)$ as given in Definition~\ref{defn1}, we have $\nabla^2 h_k(\bmtheta,\bmtheta^t)=\nabla^2 {f}_k(\bmtheta)+\lambda\mathbf{I}_d$. Substituting in the above, we have $\nabla^2 h_k(\bmtheta,\bmtheta^t)\geq\bar{\mu}\mathbf{I}_d$, which implies that  $h_k(\bmtheta,\bmtheta^t)$ is $\bar{\mu}$-strongly convex for all $t$.\\

\noindent \textbf{Proof of Lemma~\ref{lem:boundingthegradients}:}   Using the reverse triangular inequality for two vectors $\bmtheta^t$ and $\tilde{\bmtheta}^t$, and Lipschitz smoothness, we have:
\begin{align}
    \norm{\nabla{f({\bmtheta}^t)}}-\norm{\nabla{f(\tilde{\bmtheta}^t)}} &\leq\norm{\nabla{f({\bmtheta}^t)}-\nabla{f(\tilde{\bmtheta}^t)}}\nonumber\\
    &\leq L\norm{{\bmtheta}^t-\tilde{\bmtheta}^t}.
\end{align}
Since $\vecw^t = {\tilde\bmtheta}^t-{\bmtheta}^t$, we have $\norm{\nabla{f({\bmtheta}^t)}} \leq  \norm{\nabla{f(\tilde{\bmtheta}^t)}} + L\norm{\vecw^t}$. 
The same result is obtained (as an approximation) by using the Taylor series expansion as follows:
\begin{align}
f({\bmtheta^t}) = f(\tilde{\bmtheta}^t-\vecw^t) &\approx f({\tilde{\bmtheta}}^t)-\vecw^t \nabla{f(\tilde{\bmtheta}^t)}.
\end{align}
Differentiating both sides of the above equation and considering the norm, we have 
\begin{align}
    \norm{\nabla{f({\bmtheta}^t)}} &\lessapprox \norm{\nabla{f({\tilde{\bmtheta}}^t)}}+ \norm{\vecw^t}\norm{\nabla^2 f(\tilde{\bmtheta}^t)} \nonumber\\
    &\leq \norm{\nabla{f({\tilde{\bmtheta}}^t)}}+ \norm{\vecw^t}L,
\end{align}
where the last step holds by the spectral norm property, i,e., $\norm{\nabla^2 f(\tilde{\bmtheta}^t)}\leq L$ if $f$ satisfies $\norm*{\nabla f(x)-\nabla f(y)} \leq L \norm*{x-y} $.\\

\noindent \textbf{Proof of Lemma~\ref{lem:norm_gradient_h}:} The local objective function, as given in P2 is defined as follows,
\begin{align}
h_k(\bmtheta_k^{t+1};{\tilde{\bmtheta}}^t) = f_k(\bmtheta_k^{t+1}) + \frac{\lambda}{2}\norm{\bmtheta_k^{t+1} - {\tilde{\bmtheta}}^t}^2.
\label{eq:OptiProx}
\end{align}
where $\tilde{\bmtheta}^t$ is the available aggregated global model from the $t$-th aggregation epoch. From (5) in the main manuscript, we have the noisy FedAvg decoding rule given as ${\tilde\bmtheta}^t = \frac{1}{K}\sum_{k  = 1}^K \bmtheta_k^t + {\vecw}^t$. Differentiating \eqref{eq:OptiProx} with respect to $\bmtheta_k^{t+1}$, we obtain 
\begin{align}
\nabla h_k(\bmtheta_k^{t+1};\tilde{\bmtheta}^t) = \nabla{f_k(\bmtheta_k^{t+1})} + \lambda\left[\bmtheta_k^{t+1}-\tilde{\bmtheta^{t}}\right].
\label{eq:diff_prox}
\end{align}
We introduce the noiseless parameter update as  ${\bmtheta}^t = \frac{1}{K}\sum_{k=1}^{K}\bmtheta_k^t$, which leads to ${\tilde\bmtheta}^t = \bmtheta^t + {\vecw}^t$. Considering $\ell_2$ norm of both sides of the above expression, we have
\begin{align}
    \norm{\nabla{h_k(\bmtheta_k^{t+1};\tilde{\bmtheta}^t)}} &= \norm{\nabla{f_k(\bmtheta_k^{t+1})} + \lambda\left(\bmtheta_k^{t+1}-{\bmtheta}^t\right) - \lambda {\vecw}^t}.
    \end{align}
Applying the triangle inequality to the above, we obtain
    \begin{align}
\norm{\nabla{h_k(\bmtheta_k^{t+1};\tilde{\bmtheta}^t)}} &\leq \norm{\nabla{f_k(\bmtheta_k^{t+1})} + \lambda\left(\bmtheta_k^{t+1}-{\bmtheta}^t\right)} + \lambda\norm{{\vecw}^t},\nonumber\\
    \label{eq:prox_tri_inequality}
    &\leq \norm{\nabla{h_k(\bmtheta_k^{t+1};{\bmtheta}^t)}} + \lambda\norm{{\vecw}^t}.
\end{align}
Using the notion of inexactness as mentioned in Definition~\ref{defn1}, we have $\norm{\nabla{h_k(\bmtheta_k^{t+1};{\bmtheta}^t)}} \leq \gamma^t\norm{\nabla{f_k({\bmtheta}^t)}}$, the expression in \eqref{eq:prox_tri_inequality} can be rewritten as
\begin{align}
\norm{\nabla{h_k(\bmtheta_k^{t+1};\tilde{\bmtheta}^t)}} &\leq \gamma^t\norm{\nabla{f_k({\bmtheta}^t)}} + \lambda\norm{{\vecw}^t}
  \label{norm_nabla_hk}
\end{align}
Finally, using Lemma.~\ref{lem:boundingthegradients}, we have
\begin{align}
\norm{\nabla{h_k(\bmtheta_k^{t+1};\tilde{\bmtheta}^t)}} & \leq \gamma^t\norm{\nabla{f_k({\tilde{\bmtheta}}^t)}} + \gamma^t L \norm{{\vecw}^t}+\lambda\norm{{\vecw}^t}\nonumber\\
  &\leq \gamma^t\norm{\nabla{f_k(\tilde{\bmtheta}^t)}} + (\gamma^t L+\lambda) \norm{{\vecw}^t}
\end{align}

\subsubsection{\textbf{Proof of Theorem~$1$:}}
\noindent Consider the local objective function in  (15) of the main manuscript as follows,
\begin{align}
h_k(\bmtheta_k^{t+1};{\tilde{\bmtheta}}^t) = f_k(\bmtheta_k^{t+1}) + \frac{\lambda}{2}\norm{\bmtheta_k^{t+1} - {\tilde{\bmtheta}}^t}^2.
\end{align}
Denoting $\bar \bmtheta^{t+1}=\Ex_k[{\bmtheta}_k^{t+1}]$ and  differentiating the above equation and taking the expectation $\Ex_k[\cdot]$, we obtain the following:
\begin{align}
    \bar \bmtheta^{t+1} - {\bmtheta}^t = &\frac{-1}{\lambda}\Ex_k\left[\nabla{f_k(\bmtheta_k^{t+1})}\right] + \frac{1}{\lambda}\Ex_k\left[\nabla{h_k(\bmtheta_k^{t+1};\tilde{\bmtheta}^t)}\right] \nonumber\\
    &+ \vecw^t,
    \label{eq:expecthetatp1thetat}
\end{align}
where $\Ex_k[{\tilde{\bmtheta}}^t]={\bmtheta}^t+\vecw^t$.

From Lemma~\ref{lem:proxconvexity}, we know $h_k(\cdot,\cdot)$ is $\bar{\mu}$-strongly convex. Let ${\bmtheta}_k^{*,t+1}= \argmin_{\bmtheta}{h_k(\bmtheta;\bmtheta^t)}$. Using $\bar{\mu}$-strong convexity of $h_k(\cdot,\cdot)$ and \eqref{norm_nabla_hk} we obtain
\begin{align}
    \norm{{\bmtheta}^{*,t+1}_k-\bmtheta^{t+1}_k}\leq\frac{\gamma^t}{\bar{\mu}}\norm{\nabla{f_k({\bmtheta}^t)}}+\frac{\lambda}{\bar{\mu}}\norm{{\vecw}^t}.
    \label{eq:mustrong1}
\end{align}
Directly from $\bar{\mu}$-strong convexity of $h_k(\cdot)$ we have that 
\begin{align}
\norm{{\bmtheta}^{*,t+1}_k-\tilde{\bmtheta}^{t}}\leq\frac{1}{\bar{\mu}}\norm{\nabla{f_k({\tilde{\bmtheta}}^t)}}.
\label{eq:mustrong2}
\end{align}
Combining \eqref{eq:mustrong1} and \eqref{eq:mustrong2} and using triangle inequality we obtain
\begin{align}
    \norm{\bmtheta_k^{t+1}-\tilde{\bmtheta}^t} \leq \frac{\gamma^t}{\bar{\mu}}\norm{\nabla{f_k({\bmtheta}^t)}} + \frac{1}{\bar{\mu}}\norm{\nabla{f_k({\tilde{\bmtheta}}^t)}} + \frac{\lambda}{\bar{\mu}}\norm{\vecw^t}.
    \label{eq:boundnormdiffsq}
\end{align}
Substituting for $\norm{\nabla{f_k({\bmtheta}^t)}}$ from  Lemma~\ref{lem:boundingthegradients}, we obtain
\begin{align}
    \norm{\bmtheta_k^{t+1}-\tilde{\bmtheta}^t}
    &\leq \frac{1+\gamma^t}{\bar\mu}\norm{\nabla{f_k(\tilde{\bmtheta}^t)}}+ \frac{\gamma^t L+\lambda}{\bar\mu}\norm{\vecw^t}.
    \label{eq:norm_theta_k_t+1_theta_tilde_t}
\end{align}
% Using Lemma~\ref{lem4}, we have
% \begin{align}
%     \Ex\left[\norm{\vecw^t}^2\right]=\frac{d\sigma^2}{K^2p^t}\leq \frac{d\sigma^2B^2}{K^2P}\norm{\nabla F(\tilde\bmtheta^t)}^2.
% \end{align}
% Using Jensen's inequality, we can rewrite the above as
% \begin{align}
%    \Ex\left[\norm{\vecw^t}\right]\leq \sqrt{\Ex\left[\norm{\vecw^t}^2\right]}\leq \frac{\sqrt{d} \sigma B}{K\sqrt{P}}\norm{\nabla F(\tilde\bmtheta^t)}
%    \label{eq:noise2normbound}
% \end{align}
Now we bound $\norm{\bar \bmtheta^{t+1}-{\bmtheta}^t}$ from \eqref{eq:expecthetatp1thetat} as follows. 
\begin{align}
\norm{\bar \bmtheta^{t+1}-{\bmtheta}^t} & = \norm{\Ex_k[{\bmtheta}_k^{t+1}]-\Ex_k[\tilde{\bmtheta}^t] +\vecw^t} \nonumber\\
&\leq
\Ex_k\norm{\bmtheta_k^{t+1}-\tilde{\bmtheta}^t} +\norm{\vecw^t},
\label{eq:norm_thetabar_t+1_theta_t}
\end{align}
% where the last inequality is due to triangular inequality.
% Substituting the upper bound on $\norm{\bmtheta_k^{t+1}-\tilde{\bmtheta}^t}$ from  \eqref{eq:norm_theta_k_t+1_theta_tilde_t}, we obtain the following:
% \begin{align}
% &\Ex_k\norm{\bmtheta_k^{t+1}-\tilde{\bmtheta}^t}\nonumber\\
% &\leq  \left(\frac{1+\gamma^t}{\bar\mu}\right)\Ex_k[\norm{\nabla{f_k(\tilde{\bmtheta}^t)}}]+ \left(\frac{\bar \mu+\gamma^t L+\lambda}{\bar\mu}\right)\norm{\vecw^t}  \nonumber\\
% &\leq  \left(\frac{1+\gamma^t}{\bar\mu}\right)(B\norm{\nabla{F}(\tilde{\bmtheta}^t)}+H) + \left(\frac{\bar \mu+\gamma^t L+\lambda}{\bar\mu}\right)\norm{{\vecw}^t}.
% \label{eq:norm_theta_bar_t+1_theta_tilde_t}
% \end{align}
% The last inequality is due to the bounded local dissimilarity assumption, i.e.,  $~\Ex_k[\norm{\nabla{f_k}(\tilde\bmtheta)}]\leq\sqrt{~\Ex_k\norm{\nabla{f_k}(\tilde\bmtheta)}^2}\leq\norm{\nabla{F}(\tilde\bmtheta)}B+H$.
where the last inequality is due to triangular inequality.
After substituting upper bound on $\norm{\bmtheta_k^{t+1}-\tilde{\bmtheta}^t}$ from  \eqref{eq:norm_theta_k_t+1_theta_tilde_t} and using $(B,H)$-LGD, i.e.,  $~\Ex_k[\norm{\nabla{f_k}(\tilde\bmtheta)}]\leq\sqrt{~\Ex_k\norm{\nabla{f_k}(\tilde\bmtheta)}^2}\leq\norm{\nabla{F}(\tilde\bmtheta)}B+H$,we have
\begin{align}
    &\norm{\bar \bmtheta^{t+1}-{\bmtheta}^t}\nonumber\\
    &\leq  \left(\frac{1+\gamma^t}{\bar\mu}\right)\Ex_k[\norm{\nabla{f_k(\tilde{\bmtheta}^t)}}]+ \left(\frac{\bar \mu+\gamma^t L+\lambda}{\bar\mu}\right)\norm{\vecw^t}\nonumber\\
    &\leq  \left(\frac{1+\gamma^t}{\bar\mu}\right)(B\norm{\nabla{F}(\tilde{\bmtheta}^t)}+H) + \left(\frac{\bar \mu+\gamma^t L+\lambda}{\bar\mu}\right)\norm{{\vecw}^t}.
    \label{eq:norm_thetabar_t+1_theta_t_final}
\end{align}
We simplify \eqref{eq:expecthetatp1thetat} as follows: 
\begin{align}
    &\bar \bmtheta^{t+1} - {\bmtheta}^t = \frac{-1}{\lambda}\Ex_k\left[\nabla{f_k(\bmtheta_k^{t+1})}\right] + \frac{1}{\lambda}\Ex_k\left[\nabla{h_k(\bmtheta_k^{t+1};\tilde{\bmtheta}^t)}\right] + \vecw^t\nonumber\\
    &= \frac{-1}{\lambda}\left\{\Ex_k\left[\nabla f_k(\tilde{\bmtheta}^t)\right]\right.\nonumber\\
    &\left.+\Ex_k\left[\nabla{f_k(\bmtheta_k^{t+1})} -  \nabla{h_k(\bmtheta_k^{t+1};\tilde{\bmtheta}^t)}- \nabla f_k(\tilde{\bmtheta}^t)\right]\right\}+\vecw^t.
\end{align}
We define, \newline $\vecm^{t+1}\triangleq \Ex_k\left[\nabla{f_k({\bmtheta}_k^{t+1})}-\nabla{f_k(\tilde{\bmtheta}^t)-\nabla{h_k(\bmtheta_k^{t+1};{\tilde{\bmtheta}}^t)}}\right]$, which is the second term on the right hand side of the expression above. Since $\Ex_k\left[\nabla f_k(\tilde{\bmtheta}^t)\right] = \nabla F(\tilde{\bmtheta}^t)$, we have 
\begin{small}
\begin{align}
   \bar \bmtheta^{t+1} - {\bmtheta}^t = \Ex_k[{\bmtheta}_k^{t+1}] - {\bmtheta}^t = \frac{-1}{\lambda}\left(
\nabla{F(\tilde{\bmtheta}^t)}+\vecm^{t+1}\right)+\vecw^t.
\label{eq:expmt+1}
\end{align}
\end{small}
Now we derive upper bounds for the two terms on the right-hand side above. To obtain an upper bound on the norm of $\vecm^{t+1}$, we use the $L$-Lipschitz smoothness assumption, triangle inequality, upper bound on $\norm{\bmtheta_k^{t+1}-\tilde{\bmtheta}^t}$ from  \eqref{eq:norm_theta_k_t+1_theta_tilde_t} and Lemma~\ref{lem:norm_gradient_h} to obtain the following:
\begin{small}
\begin{align}
    &\norm{\vecm^{t+1}} \leq \Ex_k\left[L\norm{\bmtheta_k^{t+1}-\tilde{\bmtheta}^{t}}\right]+\Ex_k\norm{\nabla{h_k({\bmtheta_k^{t+1};{\tilde{\bmtheta}}^t)}}}\nonumber\\
    &\leq L\left[\left(\frac{1+\gamma^t}{\bar\mu}\right)B\norm{\nabla{F}(\tilde{\bmtheta}^t)} +\left(\frac{1+\gamma^t}{\bar\mu}\right)H+ \left(\frac{\gamma^t L+\lambda}{\bar\mu}\right)\norm{{\vecw}^t}\right]  \nonumber\\
    &+ \gamma^t\Ex_k[\norm{\nabla{f_k(\tilde{\bmtheta}^t)}}] + (\gamma^t L+\lambda) \norm{{\vecw}^t}.
\end{align}
    Further, using Assumption 1 to simplify $\Ex_k[\norm{\nabla{f_k(\tilde{\bmtheta}^t)}}]$ in the above expression, we have
    \begin{align}
    &\norm{\vecm^{t+1}} \leq \left[ LB\left(\frac{1+\gamma^t}{\bar\mu}\right)  + \gamma^t B\right]\norm{\nabla{F(\tilde{\bmtheta}^t)}}\nonumber\\&+ \left[ L\left(\frac{\gamma^t L+\lambda}{\bar\mu}\right) + (\gamma^t L+\lambda)\right] \norm{{\vecw}^t} +\left[L\left(\frac{1+\gamma^t}{\bar\mu}\right)+\gamma^t\right]H.
\end{align}
\end{small}
Using the Cauchy-Schwarz inequality, we know that  
\begin{align*}
    \frac{-1}{\lambda}\langle\nabla{F(\tilde{\bmtheta}^t)},\vecm^{t+1}\rangle \leq \frac{1}{\lambda} \norm{\nabla{F}(\tilde{\bmtheta}^t)} \norm{\vecm^{t+1}}.
\end{align*}
Hence, it can be shown that 
\begin{align}
&\frac{-1}{\lambda}\langle\nabla{F(\tilde{\bmtheta}^t)},\vecm^{t+1}\rangle \leq \frac{1}{\lambda}\left[ LB\left(\frac{1+\gamma^t}{\bar\mu}\right)  + \gamma^t B\right]\norm{\nabla{F(\tilde{\bmtheta}^t)}}^2\nonumber\\
&+ \frac{1}{\lambda}\left[ L\left(\frac{\gamma^t L+\lambda}{\bar\mu}\right)+ (\gamma^t L+\lambda)\right] \norm{{\vecw}^t}\norm{\nabla{F(\tilde{\bmtheta}^t)}}\nonumber\\
&+ \frac{1}{\lambda}\left[L\left(\frac{1+\gamma^t}{\bar\mu}\right)+\gamma^t\right]H \norm{\nabla{F(\tilde{\bmtheta}^t)}}.
\end{align}
Using $L$-Lipschitz smoothness of $F(\cdot)$ and Cauchy Schwartz inequality, we have 
\begin{align}
    &F(\bar \bmtheta^{t+1}) - F(\tilde{\bmtheta}^t) \leq \langle\nabla{F(\tilde{\bmtheta}^t)},\bar \bmtheta^{t+1}-{\bmtheta}^t \rangle -\langle\nabla{F(\tilde{\bmtheta}^t)},\vecw^t \rangle \nonumber\\
    &+\frac{L}{2}\norm{\bar \bmtheta^{t+1} -{\bmtheta}^t}^2+\frac{L}{2}\norm{\vecw^t}^2 - L\langle \bar \bmtheta^{t+1}-{\bmtheta}^t, \vecw^t\rangle
    \end{align}
Substituting $\bar \bmtheta^{t+1} -{\bmtheta}^t$ from \eqref{eq:expmt+1}, we obtain 
    \begin{align}
    &F(\bar \bmtheta^{t+1}) - F(\tilde{\bmtheta}^t) \leq \nabla{F(\tilde{\bmtheta}^t)}^T\left[\frac{-1}{\lambda}\left( \nabla{F(\tilde{\bmtheta}^t)}+\vecm^{t+1}\right)\right] \nonumber\\
    &+ \frac{L}{2}\norm{\bar \bmtheta^{t+1}-{\bmtheta}^t}^2+\frac{L}{2}\norm{\vecw^t}^2 + L \norm{\bar \bmtheta^{t+1}-{\bmtheta}^t}\norm{\vecw^t}
    \end{align}
    
    \noindent Substituting the bound on $\norm{\bar \bmtheta^{t+1}-{\bmtheta}^t}$ from \eqref{eq:norm_thetabar_t+1_theta_t_final}, we obtain
    % \begin{small}
    \begin{align}
    &F(\bar \bmtheta^{t+1}) - F(\tilde{\bmtheta}^t)\nonumber\\
    &\leq \tfrac{-1}{\lambda} \norm{\nabla{F(\tilde{\bmtheta}^t)}}^2 +\tfrac{L}{2} \left\{ \left(\tfrac{1+\gamma^t}{\bar\mu}\right)B\norm{\nabla{F}(\tilde{\bmtheta}^t)} \right.\nonumber\\
    &\left.+ \left(\tfrac{\bar\mu+\gamma^t L+\lambda}{\bar\mu}\right)\norm{\vecw^t} + \left(\tfrac{1+\gamma^t}{\bar\mu} \right) H\right\}^2 \nonumber\\
    &+\tfrac{1}{\lambda}\left[ LB\left(\tfrac{1+\gamma^t}{\bar\mu}\right)  + \gamma^t B\right]\norm{\nabla{F(\tilde{\bmtheta}^t)}}^2 \nonumber\\
    &+ \tfrac{1}{\lambda}\left[ L\left(\tfrac{1+\gamma^t}{\bar\mu}\right)+\gamma^t\right]H \norm{\nabla{F(\tilde{\bmtheta}^t)}} \nonumber\\
    &+ \tfrac{1}{\lambda}\left[ L\left(\tfrac{\gamma^t L+\lambda}{\bar\mu}\right) + (\gamma^t L+\lambda)\right] \norm{{\vecw}^t}\norm{\nabla{F(\tilde{\bmtheta}^t)}} \nonumber\\
    &+ \tfrac{L}{2}\norm{\vecw^t}^2+ L \left\{ \left(\tfrac{1+\gamma^t}{\bar\mu}\right)B\norm{\nabla{F}(\tilde{\bmtheta}^t)}+ \tfrac{1+\gamma^t}{\bar\mu}H \right.\nonumber\\
    & \left.+ \left(\tfrac{\bar\mu+\gamma^t L+\lambda}{\bar\mu}\right)\norm{\vecw^t}\right\} \norm{\vecw^t}.
    \label{eq:full_part_before_w_subs_S}
\end{align}
    % \end{small}
\subsubsection{Proof of Lemma~\ref{lem:precodingfactor_1storder} and  Lemma~\ref{lem:precodingfactor_zerothorder}}
This bound has been established for SGD-based updates in Lemma~A.2 in \cite{sery2021over}. Here, we assume that this bound holds for local solvers that may not use SGD.
\begin{itemize}
    \item \emph{First-order inexactness:} From the definition of $p^t$, we have $\frac{1}{p^t}\leq \frac{1}{P}\sum_{k=1}^K q_k\Ex\norm{\bmtheta_k^t-\tilde\bmtheta^t}^2$. Using \eqref{eq:norm_theta_k_t+1_theta_tilde_t}, we have
   \begin{small}
\begin{align}
    \frac{1}{p^t}&\leq\frac{\Ex_k[a\norm{\nabla f_k(\tilde\bmtheta^t)}+b\norm{\vecw^t}]^2}{P}\nonumber\\
    &\underset{\text{(i)}}{\leq}\frac{2a\Ex_k\norm{f_k(\tilde\bmtheta^t)}^2+2b \Ex\norm{\vecw^t}^2}{ P}\nonumber\\
    &\underset{\text{(ii)}}{\leq}\frac{2aK^2(B^2 \norm{\nabla F(\tilde\bmtheta)}^2 +H^2)}{K^2P-2bd \sigma^2}\nonumber\\
    &\underset{\text{(iii)}}{\leq}\frac{(B^2 \norm{\nabla F(\tilde\bmtheta)}^2 +H^2)}{P},
\end{align}
\end{small}
where $a=(1+\gamma^t)/\bar\mu$, $b=(\gamma^t L+\lambda)/\bar\mu$ and $\text{(ii)}$ is obtained using the Assumption~1. Finally, we get $\text{(iii)}$ for $\lambda\leq\frac{K^2\tau\bar\mu}{2}$, where $\tau=\frac{P}{d\sigma^2}$.
\item \emph{Zeroth-order inexactness:} From the definition of $p^t$ in \eqref{eq:precoder1} and using $\norm*{\bmtheta_k^t-\tilde \bmtheta^{t-1}}\leq \nu^t(2L\zeta^t+G)+\norm*{\vecw^{t}}$, we have 
\begin{align}
    \frac{1}{p^t}&\leq\frac{2K^2(2L\zeta^t+G)^2}{\lambda^2 (PK^2-2d\sigma^2)}.
\end{align}
For $K^2 \tau>2$, we have the result. 
\end{itemize}

Since in the case of full participation, $\vecw^t \sim \mathcal{N} (0,\frac{\sigma^2}{K^2 p^t}\mathbf{I}_d)$, therefore 
\begin{align}
    \Ex\left[\norm{\vecw^t}^2\right]=\frac{d\sigma^2}{K^2p^t}.
\end{align}
Then, using Lemma~\ref{lem:precodingfactor_1storder}, we have
\begin{align}
    \Ex\left[\norm{\vecw^t}^2\right]=\frac{ d\sigma^2}{K^2p^t}\leq \frac{ d\sigma^2}{K^2P}(B^2\norm{\nabla F(\tilde\bmtheta^t)}^2 +H^2).
       \label{eq:noise2normbound}
\end{align}
Using Jensen's inequality, we can rewrite the above as
\begin{align}
   \Ex\left[\norm{\vecw^t}\right]\leq \sqrt{\Ex\left[\norm{\vecw^t}^2\right]}\leq \frac{\sqrt{d} \sigma }{K\sqrt{P}}(B\norm{\nabla F(\tilde\bmtheta^t)}+H).
\end{align}

Furthermore, in partial participation case, $\hat\vecw^t \sim \mathcal{N} (0,\frac{\sigma^2}{\hat{K}^2 p^t}\mathbf{I}_d)$, hence 
\begin{align}
    \Ex\left[\norm{\hat\vecw^t}^2\right]=\frac{d\sigma^2}{\hat{K}^2p^t}.
       \label{eq:noise2normbound_partial}
\end{align}
Then using Lemma~\ref{lem:precodingfactor_1storder} and Jensen's inequality, we have
\begin{align}
   \Ex\left[\norm{\hat\vecw^t}\right]\leq \sqrt{\Ex\left[\norm{\hat\vecw^t}^2\right]}\leq \frac{\sqrt{d} \sigma }{\hat{K}\sqrt{P}}(B\norm{\nabla F(\tilde\bmtheta^t)}+H).
\end{align}
Similarly, for fading case, $\bar\vecw^t \sim \mathcal{N} (0,\frac{\sigma^2}{\hat r^2|\mathcal{K}^t|^2  p^t}\mathbf{I}_d)$, therefore 
\begin{align}
    \Ex\left[\norm{\bar\vecw^t}^2\right]=\frac{d\sigma^2}{\hat r^2|\mathcal{K}^t|^2  p^t}.
\end{align}
Then using Lemma~\ref{lem:precodingfactor_1storder} and Jensen's inequality, we have
\begin{align}
   \Ex\left[\norm{\bar\vecw^t}\right]\leq \sqrt{\Ex\left[\norm{\bar\vecw^t}^2\right]}\leq \frac{\sqrt{d} \sigma (B\norm{\nabla F(\tilde\bmtheta^t)}+H) }{\hat r|\mathcal{K}^t|\sqrt{P}}.
   \label{eq:noise2normbound_fading}
\end{align}
\noindent Taking expectation $\Ex[.]$ on both sides in \eqref{eq:full_part_before_w_subs_S}, rearranging the terms and subsequently using \eqref{eq:noise2normbound}, we obtain the following:
\begin{small}
\begin{align}
    \Ex[F(\bar \bmtheta^{t+1})]&\leq F(\tilde{\bmtheta}^t)-\alpha\norm{\nabla{F(\tilde{\bmtheta}^t)}}^2 +\beta H^2,
    \label{eq:Theorem1_S}
    \end{align}
    \end{small}
where 
\begin{small}
    \begin{align}
    \alpha&=\rho_1-C_1\frac{d\sigma^2}{K^2 P}, \beta=\rho_2+C_2\frac{d\sigma^2}{K^2 P}\nonumber\\
\rho_1
&= \frac{1}{\lambda}
 - \frac{\gamma^t B}{\lambda}
   -LB\left(\frac{1+\gamma^t}{\bar\mu\lambda}\right)
 - \frac{3LB^2}{2}\left(\frac{1+\gamma^t}{\bar\mu}\right)^2
 - \frac{1}{2\lambda^2}
\\[1ex]
C_1
&=
\Bigg(
 B
 \Bigg[
   L\left(\frac{\gamma^t L + \lambda}{\bar\mu}\right)
   + \gamma^t L + \gamma^t
 \Bigg]
\Bigg)^2
+ \frac{L B^2}{2}\nonumber\\
&+ \frac{3LB^2}{2}
\left(\frac{\bar\mu + \gamma^t L + \lambda}{\bar\mu}\right)^2
\\[1ex]
\rho_2
&=
\Bigg[
  L\left(\frac{1+\gamma^t}{\bar\mu}\right) + \gamma^t
\Bigg]^2
+ \frac{3L}{2}\left(\frac{1+\gamma^t}{\bar\mu}\right)^2
\\[1ex]
C_2
&=
\Bigg[
  L\left(\frac{\gamma^t L + \lambda}{\bar\mu}\right) + \gamma^t L + \lambda
\Bigg]^2
+ \frac{L}{2}
+ \frac{3L}{2}
\left(\frac{\bar\mu + \gamma^t L + \lambda}{\bar\mu}\right)^2
\end{align}
\end{small}
Now for $\alpha>0$, averaging and telescoping \eqref{eq:Theorem1_S} on both sides yields
\begin{align}
    &\frac{1}{T}\sum_{t=0}^{T-1}\Ex\norm{\nabla F(\tilde\bmtheta^{t})}^2 \leq \frac{\Delta}{\alpha T}+\frac{\beta H^2}{\alpha}\nonumber\\
    &\leq \mathcal{O}\!\left(\frac{\Delta}{\alpha T}\right)+\mathcal{O}\!\left(\frac{\beta H^2}{\alpha}\right)\leq \mathcal{O}\!\left(\frac{\Delta}{ T}\right)+\mathcal{O}\!\left(\frac{ H^2}{\alpha}\right). 
    \label{eq:BH_order_final_S}
\end{align}
The result in~\eqref{eq:BH_order_final_S} implies convergence to a neighborhood of first-order stationary points. This establishes convergence for non-convex objectives, where the neighborhood size explicitly captures the effects of heterogeneity and wireless channel impairments, while preserving the standard $\mathcal{O}(1/T)$ decay term with respect to the optimization landscape. When $H=0$, the heterogeneity-induced terms vanish, and the result reduces to $\mathcal{O}(1/T)$ convergence to exact first-order stationary points without a residual neighborhood.

\noindent  Now we consider the RoCo-FL algorithm when only a subset of the clients participating in the FL process, i.e., $\hat K$ clients, are chosen randomly for federation. 

\noindent Due to the presence of transmission noise, the received aggregated model $\tilde{\bmtheta}^{t+1}$ is perturbed by an additive Gaussian noise. We can quantify the expected deviation as $\Ex\norm{\tilde{\bmtheta}^{t+1} - \bar{\bmtheta}^{t+1}} = \Ex\norm{\vecw^{t+1}} \leq \frac{\sqrt{d}\sigma}{K\sqrt{p^{t+1}}} = \xi'$, i.e., the perturbed model remains within a small neighborhood of the true model in expectation. The Lipschitz continuity of $F(\cdot)$ is formally defined over a deterministic neighborhood of $\bar{\bmtheta}^{t+1}$ and holds if the deviation satisfies $\norm{\tilde{\bmtheta}^{t+1} - \bar{\bmtheta}^{t+1}} \leq \xi$ for some local radius $\xi > 0$. To apply this deterministic Lipschitz condition despite stochastic noise, we argue as follows: if the expected deviation $\Ex\norm{\tilde{\bmtheta}^{t+1} - \bar{\bmtheta}^{t+1}}$ is strictly less than $\xi'$, then the perturbation remains within the Lipschitz neighborhood in expectation. While this does not guarantee that every realization of $\tilde{\bmtheta}^{t+1}$ lies within the neighborhood, it ensures that the dominant contribution to the expected function difference comes from regions where the local Lipschitz property holds. Applying the local Lipschitz condition, we have
\begin{align}
    F({\tilde\bmtheta}^{t+1}) \leq F(\bar{\bmtheta}^{t+1})+L_0 \norm{\tilde\bmtheta^{t+1}-\bar\bmtheta^{t+1}},
    \label{eq:Thm2firststepL0}
\end{align}
where $L_0$ is the local Lipschitz constant. Considering $\Ex_{\mathcal{S}^t}[.]$ on both sides of \eqref{eq:Thm2firststepL0}, we obtain
\begin{align}
    \Ex_{S^t}[F(\tilde\bmtheta^{t+1})] \leq \Ex_{S^t}[F(\bar\bmtheta^{t+1})] +q^t,
    \label{eq:Theorem2secondstep_qt}
\end{align}
where $q^t=\Ex_{S^t}[L_0 \norm{\tilde\bmtheta^{t+1}-\bar\bmtheta^{t+1}}]$. Evidently, we need to obtain an upper bound on the expected norm of $q^t$ so that the expected decrease and the rate of decrease in the loss function can be quantified. Towards this, we use the bound $L_0$ as given in \cite{li2020federated}, i.e,
\begin{align}
        L_0 &\leq \norm{\nabla F(\bmtheta^{t})} + L \left(\norm{\bar\bmtheta^{t+1}-\bmtheta^{t}}+\norm{\tilde\bmtheta^{t+1}-\bmtheta^{t}}\right) 
\end{align}
Using the above result in $q^t=\Ex_{S^t}\left[L_0 \norm{\tilde\bmtheta^{t+1}-\bar\bmtheta^{t+1}}\right]$, the upper bound on the $q^t$ is given as
\begin{align}
    q^t &\leq \Ex_{S^t}\Big[ \underbrace{\left\{\norm{\nabla F(\bmtheta^{t})}+L \left(\norm{\bar\bmtheta^{t+1}-\bmtheta^{t}}+\norm{\tilde\bmtheta^{t+1}-\bmtheta^{t}}\right)\right\}}_{\geq L_0} \nonumber\\
    &\quad\quad\times\norm{\tilde\bmtheta^{t+1}-\bar\bmtheta^{t+1}} \Big].
\end{align}
Using Lemma~\ref{lem:boundingthegradients} in the context of $F({\bmtheta}^t)$, we obtain the following
\begin{small}
    \begin{align}
    q^t & {\leq} \Ex_{S^t}\left[ \left\{\norm{\nabla{F({\tilde\bmtheta}^t)}} +L\norm{\hat\vecw^t}+L \left(\norm{\bar\bmtheta^{t+1}-\bmtheta^{t}}+\norm{\tilde\bmtheta^{t+1}-\bmtheta^{t}}\right)\right\}\right. \nonumber\\
    &\left.\norm{\tilde\bmtheta^{t+1}-\bar\bmtheta^{t+1}} \right]\nonumber\\
    &\leq \left(\norm{\nabla{F({\tilde\bmtheta}^t)}} +L\norm{\hat\vecw^t}+ L\norm{{\bar\bmtheta}^{t+1}-{\bmtheta}^t}\right) \Ex_{S^t}\norm{{\tilde\bmtheta}^{t+1}-\bar{\bmtheta}^{t+1}}\nonumber\\&\qquad\qquad+L\Ex_{S^t}\left[\norm{{\tilde\bmtheta}^{t+1}-{\bmtheta}^{t}}\norm{{\tilde\bmtheta}^{t+1}-\bar{\bmtheta}^{t+1}}\right]\nonumber\\
    &\underset{\text{(i)}}{\leq} \left(\norm{\nabla{F({\tilde\bmtheta}^t)}} +L\norm{\hat\vecw^t}+ L\norm{{\bar\bmtheta}^{t+1}-{\bmtheta}^t}\right) \Ex_{S^t}\norm{{\tilde\bmtheta}^{t+1}-\bar{\bmtheta}^{t+1}} 
    \nonumber\\
    &+L\Ex_{S^t}\left[\left(\norm{{\tilde\bmtheta}^{t+1}-{\bar\bmtheta}^{t+1}}+\norm{{\bar\bmtheta}^{t+1}-{\bmtheta}^{t}}\right)\norm{{\tilde\bmtheta}^{t+1}-\bar{\bmtheta}^{t+1}}\right], 
    \end{align}
\end{small}

where $\text{(i)}$ holds by the triangular inequality (applied as $\norm{a-b}\leq\norm{a-c}+\norm{c-b}$). Rearranging the terms above, we see that 
    \begin{align}
    q^t &\leq\left(\norm{\nabla{F({\tilde\bmtheta}^t)}} +L\norm{\hat\vecw^t}+ 2L\norm{{\bar\bmtheta}^{t+1}-{\bmtheta}^t}\right)\nonumber\\ &\Ex_{S^t}\norm{{\tilde\bmtheta}^{t+1}-\bar{\bmtheta}^{t+1}}+L\Ex_{S^t}\norm{{\tilde\bmtheta}^{t+1}-\bar{\bmtheta}^{t+1}}^2.
   \label{eq:qt1}
    \end{align}
We now consider upper bounds for individual terms in the above expression \eqref{eq:qt1}. First, we consider $\Ex_{S^t}\norm{{\tilde\bmtheta}^{t+1}-\bar{\bmtheta}^{t+1}} \leq \sqrt{\Ex_{S^t}\norm{{\tilde\bmtheta}^{t+1}-\bar{\bmtheta}^{t+1}}^2}$ and subsequently upper bound $\Ex_{S^t}\left[\norm{{\tilde\bmtheta}^{t+1}-\bar{\bmtheta}^{t+1}}^2\right]$ as follows:
\begin{align}
&\Ex_{S^t}\left[\norm{{\tilde\bmtheta}^{t+1}-\bar{\bmtheta}^{t+1}}^2\right] = \Ex_{S^t}\left[\norm*{\frac{1}{\hat{K}}\sum_{k=1}^{\hat{K}}\bmtheta_{k}^{t+1}+ \hat\vecw^t-\bar{\bmtheta}^{t+1}}^2\right]\nonumber\\
    &\underset{\text{(i)}}{\leq}\frac{1}{(\hat{K})^2} \sum_{k=1}^{\hat{K}} \Ex_{S^t}[\norm{\bmtheta_{k}^{t+1}-\bar{\bmtheta}^{t+1}}^2]+ \norm{\hat\vecw^t}^2 \nonumber\\
    &\qquad\qquad+ 2 \Ex_{S^t} \langle\bmtheta_{k}^{t+1}-\bar{\bmtheta}^{t+1},\hat\vecw^t \rangle \nonumber\\
    &\underset{\text{(ii)}}{\leq} \frac{1}{\hat{K}}\Ex_{k}[\norm{\bmtheta_{k}^{t+1}-\bar{\bmtheta}^{t+1}}^2]+ \norm{\hat\vecw^t}^2\nonumber\\
    &\underset{\text{(iii)}}{\leq} \frac{1}{\hat{K}}\Ex_{k}[\norm{(\bmtheta_{k}^{t+1} -\tilde\bmtheta^t) -(\bar{\bmtheta}^{t+1}-\tilde\bmtheta^t) }^2]+ \norm{\hat\vecw^t}^2\nonumber\\
    &\underset{\text{(iv)}}{\leq} \frac{2}{\hat{K}}\Ex_{k}[\norm{(\bmtheta_{k}^{t+1} -\tilde\bmtheta^t)}^2]+ \norm{\hat\vecw^t}^2
    \end{align}
where $\text{(i)}$ follows from Jensen's inequality, $\text{(ii)}$ is derived using Lemma~$4$ in \cite{li2019convergence} and $\Ex_{S^t} \langle\bmtheta_{k}^{t+1}-\bar{\bmtheta}^{t+1},\hat\vecw^t \rangle =0$. We add and subtract $\tilde{\bmtheta}^{t}$ in $\text{(iii)}$ and finally we arrive at $\text{(iv)}$ since  $\Ex_k\left[\bmtheta^{t+1}_{k}\right] = \bar{\bmtheta}^{t+1}$.
% \vspace{-1cm}
\begin{small}
\begin{align}
    &\Ex_{S^t}\left[\norm{{\tilde\bmtheta}^{t+1}-\bar{\bmtheta}^{t+1}}^2\right] \leq \frac{2}{\hat K}\Ex_{k}[\norm{\bmtheta_{k}^{t+1}-\tilde{\bmtheta}^{t}}^2]+ \norm{\hat\vecw^t}^2\nonumber\\
    &\underset{\text{(v)}}{\leq} \frac{2}{\hat K}\Ex_{k} \left[\left(\frac{1+\gamma^t}{\bar{\mu}}\right)\norm{\nabla f_k(\tilde \bmtheta^t)}+ \left(\frac{\gamma^t L+\lambda}{\bar{\mu}}\right)\norm{\hat\vecw^t}\right]^2+ \norm{\hat\vecw^t}^2\nonumber\\
    &\underset{\text{(vi)}}{\leq} \frac{2}{\hat K}\left[ \left(\frac{1+\gamma^t}{\bar{\mu}}\right)B\norm{\nabla{F({\tilde\bmtheta}^t)}}+\left(\frac{1+\gamma^t}{\bar\mu}\right)H \right. \nonumber\\
    &\left.+\left(\frac{\gamma^t L+\lambda}{\bar\mu}\right)\norm{\hat\vecw^t} \right]^2 + \norm{\hat\vecw^t}^2,
    \label{eq:S^to_k}
  \end{align}
  \end{small}
  % \vspace{-1cm}
 where \eqref{eq:norm_theta_k_t+1_theta_tilde_t} and Assumption~$1$ yields inequalities $\text{(v)}$ and $\text{(vi)}$ respectively. We complete the upper bound on $q^t$ by substituting and thereafter adjusting the bounds from \eqref{eq:norm_thetabar_t+1_theta_t_final} and \eqref{eq:S^to_k} in \eqref{eq:qt1} and we get
\begin{small}
\begin{align}
    &q^t \leq \nonumber\\
    &\left[\norm{\nabla{F({\tilde\bmtheta}^t)}} +L\norm{\hat\vecw^t}+ 2L \left\{ \left(\frac{1+\gamma^t}{\bar\mu}\right)(B\norm{\nabla{F}(\tilde{\bmtheta}^t)}+H)\right. \right.\nonumber\\
    &\left.\left.+ \left(\frac{\bar\mu+\gamma^t L+\lambda}{\bar\mu}\right)\norm{\hat\vecw^t}\right\}\right] \times \nonumber\\
    &\Bigg[\frac{\sqrt{2}}{\sqrt{\hat K}}\left\{ \left(\frac{1+\gamma^t}{\bar{\mu}}\right)(B\norm{\nabla{F({\tilde\bmtheta}^t)}}+H)+ \left(\frac{\gamma^t L+\lambda}{\bar\mu}\right)\norm{\hat\vecw^t} \right\}\nonumber\\
    &+ \norm{\hat\vecw^t} \Bigg]+ L \Bigg[\frac{2}{\hat K}\left\{ \left(\frac{1+\gamma^t}{\bar{\mu}}\right)(B\norm{\nabla{F({\tilde\bmtheta}^t)}}+H)\right. \nonumber\\
    &\left.+ \left(\frac{\gamma^t L+\lambda}{\bar\mu}\right)\norm{\hat\vecw^t} \right\}^2  + \norm{\hat\vecw^t}^2 \Bigg].
  \label{eq:qt_before_w_subs_S}
\end{align}
\end{small}
Now substituting \eqref{eq:qt_before_w_subs_S}
in \eqref{eq:full_part_before_w_subs_S} and then taking expectation with respect to $\hat\vecw^t$ and using \eqref{eq:noise2normbound_partial}, we get
\begin{small}
\begin{align}
    \Ex_{\mathcal{S}^t}[F({\tilde\bmtheta}^{t+1})]
    &\leq F({\tilde\bmtheta}^{t})-  \hat\alpha \norm{\nabla{F({\tilde\bmtheta}^t)}}^2 +\hat\beta H^2,
    \label{eq:partial_participation_BH_before_youngs_S}
\end{align}
\end{small}
where
\begin{small}
\begin{align}
&\hat\alpha=\hat \rho_1-\hat C_1\frac{d\sigma^2}{\hat K^2 P}, \hat \beta=\hat \rho_2+\hat C_2\frac{d\sigma^2}{\hat K^2 P}\nonumber\\
\hat{\rho}_1&=
\frac{1}{\lambda}
- \frac{\gamma^t B}{\lambda}
- \frac{1}{2\lambda^2}-1
- \left(\frac{1+\gamma^t}{\bar\mu}\right) \left(\frac{LB}{\lambda}+\frac{\sqrt{2}B}{\sqrt{\hat K}}+4LB\right) \nonumber\\
&-\left(\frac{1+\gamma^t}{\bar\mu}\right)^2
\left(
\frac{3LB^2}{2}
+ \frac{2\sqrt{2}L B^2}{\sqrt{\hat K}}
+ 4L^2B^2
+ \frac{2B^2}{\hat K}
+ \frac{8LB^2}{\hat K}
\right)\nonumber\\
\hat{C}_1
&=
\Bigg(
 B
 \Bigg[
   L\left(\frac{\gamma^t L + \lambda}{\bar\mu}\right)
   + \gamma^t L + \gamma^t
 \Bigg]
\Bigg)^2
+ \frac{L B^2}{2}\nonumber\\
&+ \frac{3LB^2}{2}
\left(\frac{\bar\mu + \gamma^t L + \lambda}{\bar\mu}\right)^2
+ B^2\Bigg(L+ \frac{\bar\mu+\gamma^t L+\lambda}{\bar\mu}\Bigg)^2\nonumber\\
&
+ B^2
\left(
\frac{\sqrt{2}}{\sqrt{\hat K}}
\left(\frac{\gamma^t L + \lambda}{\bar\mu}\right)
+ 1
\right)^2
+ B^2
\left(1+
\frac{4L(\gamma^t L + \lambda)^2}{\hat K \bar\mu}
\right)^2\nonumber\\
&+ B^2\Bigg(L+ \frac{\bar\mu+\gamma^t L+\lambda}{\bar\mu}\Bigg)\left(
\frac{\sqrt{2}}{\sqrt{\hat K}}
\left(\frac{\gamma^t L + \lambda}{\bar\mu}\right)
+ 1
\right) \nonumber\\
\hat{\rho}_2
&=
\left(\frac{1+\gamma^t}{\bar\mu}\right)^2
\left(
L
+ \frac{3L}{2}
+ 4L^2
+ \frac{2}{\hat K}+\frac{2\sqrt{2}}{\sqrt{\hat K}}
+ \frac{8L}{\hat K}
\right)\nonumber\\
&+ (\gamma^t)^2
+ 2\gamma^t L\left(\frac{1+\gamma^t}{\bar\mu}\right)
\nonumber\\
\hat{C}_2
&=
\Bigg[
L\left(\frac{\gamma^t L + \lambda}{\bar\mu}\right)
+ \gamma^t L + \lambda
\Bigg]^2
+ \frac{L}{2}
+ \frac{3L}{2}
\left(\frac{\bar\mu + \gamma^t L + \lambda}{\bar\mu}\right)^2
\nonumber\\
&
+ \left(
L + \frac{\bar\mu + \gamma^t L + \lambda}{\bar\mu}
\right)^2+ \left\{
\frac{\sqrt{2}}{\sqrt{\hat K}}
\left(\frac{\gamma^t L + \lambda}{\bar\mu}\right)
+ 1
\right\}^2\nonumber\\
&+ \left(
L + \frac{\bar\mu + \gamma^t L + \lambda}{\bar\mu}
\right)
\left(
\frac{\sqrt{2}}{\sqrt{\hat K}}
\left(\frac{\gamma^t L + \lambda}{\bar\mu}\right)
+ 1
\right)\nonumber\\
&
+ \left\{
1
+ \frac{4L}{\hat K}
\left(\frac{\gamma^t L + \gamma^t}{\bar\mu}\right)^2
\right\}
\end{align}
\end{small}

\noindent \emph{Rate of convergence: }
Averaging and telescoping \eqref{eq:partial_participation_BH_before_youngs_S} on both sides, we get
\begin{align}
    &\frac{1}{T}\sum_{t=0}^{T-1}\Ex\norm{\nabla F(\tilde\bmtheta^{t})}^2 \leq \frac{\Delta}{\hat\alpha T}+\frac{\hat\beta H^2}{\hat\alpha}\nonumber\\
    &\leq \mathcal{O}\!\left(\frac{\Delta}{\hat\alpha T}\right)+\mathcal{O}\!\left(\frac{\hat\beta H^2}{\hat\alpha}\right) \leq \mathcal{O}\!\left(\frac{\Delta}{ T}\right)+\mathcal{O}\!\left(\frac{H^2}{\hat\alpha}\right). 
\end{align}
The bound consists of a transient term that decays as $\mathcal{O}(1/T)$ and a residual term proportional to $H^2$, implying convergence to a neighborhood of first-order stationary points. The neighborhood radius vanishes when $H=0$. 
% When expressed in terms of the $BH$ parameters, this residual can be interpreted as $\mathcal{O}(H^2/B^2)$, highlighting the role of gradient alignment in mitigating heterogeneity-induced bias.\\

%%%%%FADING
\subsubsection{Fading} 
\textbf{Proof of Corollary~\ref{fading_theorem_BH}} 
Following the proof steps similar to Theorem~\ref{partial_participation_theorem_BH_M} and taking expectation with respect to $\bar \vecw^t$ in the final step, we obtain
\begin{small}
\begin{align}
    \Ex_{\mathcal{K}^t}[F({\tilde\bmtheta}^{t+1})]
    &\leq F({\tilde\bmtheta}^{t})-  \bar\alpha \norm{\nabla{F({\tilde\bmtheta}^t)}}^2 + \bar\beta H^2.
    \label{eq:fading_BH_before_youngs_S}
\end{align}
\end{small}
where 
\begin{small}
\begin{align} 
\bar\alpha&=\hat \rho_1-\hat C_1\frac{d\sigma^2}{\hat r^2\hat K^2 P},\quad \bar \beta=\hat \rho_2+\hat C_2\frac{d\sigma^2}{\hat r^2\hat K^2 P}.
\end{align}
\end{small} 
Finally, averaging and telescoping \eqref{eq:fading_BH_before_youngs_S} on both sides, we get
\begin{align}
    &\frac{1}{T}\sum_{t=0}^{T-1}\Ex\norm{\nabla F(\tilde\bmtheta^{t})}^2 \leq \frac{\Delta}{\bar\alpha T}+\mathcal{O}\!\left(\frac{\bar\beta H^2}{\bar\alpha}\right)\nonumber\\
    &\leq \mathcal{O}\!\left(\frac{\Delta}{\bar\alpha T}\right)+\mathcal{O}\!\left(\frac{H^2}{B^2}\right). 
\end{align}

\begin{prop}[Local computation implies $\gamma_k^t$-inexactness]
\label{prop:sgd_inexact}
Let Assumptions~2--3 hold and let $\bar{\mu} = \lambda - \bar{L} > 0$. 
Suppose client $k$ initializes at $\tilde{\theta}^{t-1}$ and runs $E_k$ steps 
of stochastic gradient descent on $h_k(\theta;\tilde{\theta}^{t-1})$ with 
step size $\eta \leq 1/L$ and unbiased stochastic gradients with bounded variance.
Then the resulting iterate $\theta_k^t$ satisfies the inexactness condition in Definition~\ref{defn1} with
\begin{align}
\gamma_k^t = (1 - \eta\bar{\mu})^{E_k}.
\label{eq:gamma_explicit}
\end{align}
\end{prop}
\begin{proof}
Under Assumptions~2--3, the local objective 
$h_k(\theta;\tilde{\theta}^{t-1})$ is $L$-smooth and 
$\bar{\mu}$-strongly convex. 
For stochastic gradient descent with step size $\eta \leq 1/L$, standard results imply that the iterates contract toward a neighborhood of the minimizer, with a rate governed by $(1-\eta\bar{\mu})^{E_k}$ and a residual error due to stochastic gradient noise \cite{nesterov2004}.
At initialization $\theta_k^{(0)} = \tilde{\theta}^{t-1}$, we have
\[
\nabla h_k(\tilde{\theta}^{t-1};\tilde{\theta}^{t-1})
= \nabla f_k(\tilde{\theta}^{t-1}).
\]
Thus, after $E_k$ steps, the gradient norm at $\theta_k^t$ can be decomposed into:
(i) a contraction component proportional to $(1-\eta\bar{\mu})^{E_k}\|\nabla f_k(\tilde{\theta}^{t-1})\|$, and 
(ii) a residual term capturing stochastic gradient noise and OTA perturbations.
Comparing with Definition~\ref{defn1} yields the result.
\end{proof}
Proposition~\ref{prop:sgd_inexact} shows that 
the $\gamma_k^t$-inexactness condition in 
Definition~\ref{defn1} is not an assumption imposed on 
the algorithm, but a guarantee automatically 
delivered by running $E_k$ SGD steps. 
This closes the gap between the theoretical 
framework and Algorithm~\ref{tab:AlgTable}. 
Furthermore,~\eqref{eq:gamma_explicit} 
gives a precise handle on the 
computation-inexactness trade-off: 
a straggler with $E_k = 1$ achieves 
$\gamma_k^t = 1 - \eta\bar{\mu}$, 
close to 1 for small $\eta$, 
indicating a highly inexact solution, 
while a full participant with $E_k = E$ 
achieves $\gamma_k^t = (1-\eta\bar{\mu})^E$, 
decreasing geometrically with $E$. 
Since Definition~1 already accommodates 
per-client $\gamma_k^t$, this heterogeneity 
in local computation is naturally absorbed 
into the convergence framework: 
clients contributing more local steps 
produce more exact updates, 
while stragglers contribute less exact but still valid updates, 
all within the same theoretical guarantee.

\subsection{Detailed Proofs of Key Lemmas and Theorems under $G$-Lipschitz assumption}
\subsubsection{\textbf{Proof of Theorem~$3$}}
The proximal-constraint-based optimization problem can be written as:
\begin{equation}
    \bmtheta_k^t=\argmin \{h_k(\bmtheta;\tilde\bmtheta^{t-1})=f_k(\bmtheta)+\frac{1}{2\nu^t}\norm*{\bmtheta-\tilde\bmtheta^{t-1}}^2\},
\end{equation}
where $\nu^t$ acts as a step size of the proximal subproblem.\\
   \emph{From Definition~\ref{defn1}, for a differentiable function $h_k(\bmtheta;\tilde\bmtheta^{t-1})=f_k(\bmtheta)+\frac{1}{2\nu^t}\norm{\bmtheta-\tilde\bmtheta^{t-1}}^2$, 
$\bmtheta^t_k$ is a $\zeta^t$-inexact solution of $\min\limits_{\bmtheta}h_k(\bmtheta;\tilde\bmtheta^{t-1})$ if $h_k(\bmtheta^t_k;\tilde\bmtheta^{t-1})\leq\min\limits_{\bmtheta}h_k(\bmtheta,\tilde\bmtheta^{t-1})+\zeta^t$,  for $\zeta^t\geq0$.}

% \begin{equation} 
%  h_k(\bmtheta_k^t;\tilde \bmtheta^{t-1})\leq \min_{\bmtheta} h_k(\bmtheta;\tilde\bmtheta^{t-1})+\zeta^t
% \end{equation}
\noindent Let us define $\phi_k^t=\nabla f_k(\bmtheta_k^t)$, then
$\phi^t=\frac{1}{\hat K} \sum_{k\in K^t} \phi_k^t $ and $ \bar{\phi}^t=\frac{1}{K}\sum_{k=1}^K \phi_k^t$. 
\begin{lemma}
 \label{G_lemma}
    Assuming the loss function is G-Lipschitz for each $k\in [K]$, it holds that 
    \begin{equation}
        \Ex[\phi^t]=\bar{\phi}^t, \quad \Ex[\norm*{\phi^t-\bar{\phi}^t}^2]\leq\frac{G^2}{\hat K}
    \end{equation}
\end{lemma}
\emph{Proof.}
\begin{small}
\begin{align}
  &\Ex[\norm*{\phi^t-\bar\phi^t}^2]=\Ex\norm*{\frac{1}{\hat K}\sum_{k\in\hat K} \phi_k^t-\bar\phi^t}^2\nonumber\\
  &=\frac{1}{\hat K^2}\sum_{k\in\hat K}\Ex\norm*{ \phi_k^t-\bar\phi^t}^2\leq \frac{1}{\hat K}\Ex\norm*{\phi_k^t}^2\leq\frac{G^2}{\hat K}
\end{align}
\end{small}

\begin{lemma}
    Given  $\zeta^t$-inexactness, setting $\nu^t\leq\frac{1}{L}$ and assuming the loss function is L-smooth with respect to its first argument for each $k\in [K]$, it holds that
    \begin{align}
        &\norm*{\bmtheta_k^t-\tilde \bmtheta^{t-1}+\nu^t \phi_k^t}=\norm*{\bmtheta_k^t-\bmtheta^{t-1}-\hat\vecw^{t-1}+\nu^t \phi_k^t} \nonumber \\
        &\leq \norm*{\bmtheta_k^t-\bmtheta^{t-1}+\nu^t \phi_k^t}+\norm*{\hat\vecw^{t-1}}  \leq 2L\zeta^t\nu^t +\norm*{\hat\vecw^{t-1}}
    \end{align}
\end{lemma}
From above Lemma 
\begin{equation}
    \norm*{\bmtheta_k^t-\tilde \bmtheta^{t-1}}\leq \nu^t(2L\zeta^t-G)+\norm*{\hat\vecw^{t}},
\end{equation}
where we assume $\norm*{\hat\vecw^{t-1}}=\norm*{\hat\vecw^{t}}$ for simplicity.
\begin{lemma}
\label{norm_nablaF-phibar}
Given  $\zeta^t$-inexactness, setting $\nu^t\leq\frac{1}{L}$ and assuming the loss function is G-Lipschitz and L-smooth with respect to its first argument for each $k\in [K]$, it holds that
    \begin{align}
     \norm*{\nabla F(\tilde \bmtheta^{t-1})-\bar{\phi}^t}^2 \leq L^2\{(2L\zeta^t -G) \nu^t +\norm*{\hat\vecw^{t}}\}^2
    \end{align}
\end{lemma}
\begin{proof}
    \begin{align}
         &\norm*{\nabla F(\tilde \bmtheta^{t-1})-\bar{\phi}^t}^2 \leq \norm*{\frac{1}{ K} \sum_{k=1}^K\left(\nabla f_k(\tilde\bmtheta^{t-1})-\nabla f_k(\bmtheta_k^t)\right)}^2\nonumber\\
         &\leq \frac{L^2}{K}\sum_{k=1}^K\norm*{\tilde\bmtheta^{t-1}-\bmtheta_k^t}^2\leq L^2\{(2L\zeta^t -G) \nu^t +\norm*{\hat\vecw^{t}}\}^2.
    \end{align}
\end{proof}
Using Lemma~\ref{lem:precodingfactor_zerothorder}, we have
\begin{align}
    \Ex\left[\norm{\hat\vecw^t}^2\right]&=\frac{d\sigma^2}{\hat K^2p^t} \leq \frac{d\sigma^2G^2 (\nu^t)^2}{\hat K^2 P}
    \label{eq:noise_G}
\end{align}
Let us define $\delta_k^t=\frac{1}{\nu^t}(\bmtheta_k^t - \tilde\bmtheta^{t-1}) + \phi_k^t$, $\delta^t=\frac{1}{\hat K} \sum_{k\in K^t} \delta_k^t$ and $\bar{\delta}^t=\frac{1}{ K} \sum_{k=1}^K \delta_k^t$.
  Then we have $\Ex[\delta^t]=\bar \delta ^t$ and $\tilde\bmtheta^t=\tilde\bmtheta^{t-1}-\nu^t(\phi^t-\delta^t)+\hat\vecw^t$.

 Also, it follows from the Lemma and the triangular inequality that 
 \begin{equation}
     \max \{\norm*{\bar\delta^t},\norm*{\delta^t}\} \leq 2L\zeta^t+\norm*{\hat\vecw^t}.
     \label{eq:max_normdelta}
 \end{equation}
 Now, since the loss is $L$-Smooth, we have 
 \begin{small}
 \begin{align}
     &\Ex[ F(\tilde\bmtheta^t)]\nonumber\\
     &\leq \Ex\bigg[ F(\tilde\bmtheta^{t-1}) + \langle \nabla  F(\tilde\bmtheta^{t-1}),\tilde\bmtheta^t-\tilde\bmtheta^{t-1}\rangle + \frac{L}{2}\norm*{\tilde\bmtheta^t-\tilde\bmtheta^{t-1}}^2\bigg] \nonumber\\
     &\leq \Ex\bigg[ F(\tilde\bmtheta^{t-1})-\nu^t\langle\nabla F(\tilde\bmtheta^{t-1}),\phi^t-\delta^t\rangle\bigg]\nonumber\\
     &+\frac{L}{2}\Ex\bigg[\norm*{-\nu^t(\phi^t-\delta^t)+\hat \vecw^t}^2 \bigg]   \nonumber\\
     &\leq  F(\tilde\bmtheta^{t-1})-\nu^t\langle\nabla  F(\tilde\bmtheta^{t-1}),\bar\phi^t-\bar\delta^t\rangle\nonumber\\
     &+\frac{L}{2}\Ex\bigg[(\nu^t)^2\norm*{\phi^t-\delta^t}^2 + \norm*{\hat\vecw^t}^2+ 2 \langle-\nu^t(\phi^t-\delta^t), \hat\vecw^t \rangle\bigg]   \nonumber\\
     &\underset{\text{(i)}}{\leq} F(\tilde\bmtheta^{t-1})-\frac{\nu^t}{2}\norm*{\nabla F(\tilde\bmtheta^{t-1})}^2-\frac{\nu^t}{2}\norm*{\bar\phi^t}^2+\frac{\nu^t}{2} \norm*{\nabla F(\tilde \bmtheta^{t-1}) - \bar \phi^t}^2 \nonumber\\
     &+\nu^t G(2L\zeta^t+\norm*{\hat\vecw^t})+\frac{3L(\nu^t)^2}{2}\norm*{\bar \phi^t}^2+ \frac{3L(\nu^t)^2 G^2}{2\bar K}\nonumber\\
     &+\frac{3L(\nu^t)^2}{2}\norm*{\delta^t}^2 + \frac{L}{2}\norm*{\hat\vecw^t}^2 \nonumber\\
 & \underset{\text{(ii)}}{\leq}F(\tilde\bmtheta^{t-1}) - \frac{\nu^t}{2}\norm*{\nabla F(\tilde\bmtheta^{t-1})}^2 + \frac{6LG^2(\nu^t)^2}{\hat K}+\frac{c_w\norm*{\hat\vecw^t}^2}{2}.
   \end{align}
 \end{small}
 \noindent where $c_w=3L(\nu^t)^2+L^2\nu^t+1+L$, \text{(i)} is obtained using G-Lipschitz of loss function, triangular inequality, Lemma~\ref{G_lemma}, \eqref{eq:max_normdelta}, $\nu^t\leq\frac{1}{3L}$, Lemma~\ref{norm_nablaF-phibar} and in $\text{(ii)}$, first we use $\zeta^t\leq \min\bigg\{\tfrac{G}{2L\sqrt{\hat K}},\tfrac{G\nu^t}{\hat K},\tfrac{G}{2L}\bigg\}$ and then use Young's inequality.\\
 Now taking expectation with respect noise and using \eqref{eq:noise_G}, we get
 \begin{small}   
\begin{align}
&F(\tilde\bmtheta^t)\nonumber\\
 & \leq F(\tilde\bmtheta^{t-1}) - \frac{\nu^t}{2}\norm*{\nabla F(\tilde\bmtheta^{t-1})}^2 + \frac{6LG^2(\nu^t)^2}{\hat K}+ \frac{ c_w G^2(\nu^t)^2d\sigma^2}{\hat K^2 P}.
 \end{align}
   \end{small}

Rearranging the terms and taking the expectation over random iterates, we get
\begin{small}
    \begin{align}
    &\Ex[\norm*{\nabla F(\tilde\bmtheta^{t-1})}^2]\nonumber\\
    & \leq\frac{2}{\nu^t}\Ex[F(\tilde\bmtheta^{t-1})-F(\bmtheta^{t})] + \frac{12 LG^2\nu^t}{\hat K}+\frac{c_w G^2(\nu^t) d\sigma^2}{\hat K^2 P}.
\end{align}
\end{small}

Now averaging over $t=1...T$ yields
\begin{small}
\begin{align}
   & \frac{1}{T}\sum_{t=0}^{T-1} \Ex[\norm*{\nabla F(\tilde\bmtheta^{t})}^2] \leq \frac{2\Delta}{ T\nu^t}+ \frac{12 LG^2\nu^t}{\hat K}+\frac{c_w G^2(\nu^t) d\sigma^2}{\hat K^2 P}.
\end{align}
\end{small}

Setting $\nu^t \equiv \nu$ and using $\nu=\frac{1}{3L}\sqrt{\frac{\hat K}{T}}$, we get
\begin{align}
&\frac{1}{T}\sum_{t=0}^{T-1} \Ex[\norm*{\nabla F(\tilde\bmtheta^{t})}^2]\leq \frac{L\Delta+G^2}{\sqrt{T\hat K}}+\frac{c_w G^2 d\sigma^2}{3L \hat K\sqrt{T \hat K} P}
\end{align}
\begin{align}
&\frac{1}{T}\sum_{t=0}^{T-1}\Ex\left[\norm*{\nabla F(\tilde\bmtheta^{t})}^2\right] 
\leq\frac{L\Delta+\tilde C}{\sqrt{T\hat K}},
\label{eq:G_theorem_eq}
\end{align}
where $\tilde C= G^2\left(1+\frac{ d\sigma^2}{3 L P \hat K}(3 L\nu^2+L^2\nu+1+L)\right)$ and $\Delta=F(\bmtheta^0)-F(\bmtheta^T)$.

\end{document}